\documentclass[pdflatex,sn-mathphys-num]{sn-jnl}

\usepackage[title]{appendix}%
\usepackage{textcomp}%
\usepackage{manyfoot}%
\usepackage{algorithmicx}%
\usepackage{listings}%
\usepackage{lmodern}

\usepackage{amsfonts} 
\usepackage{amsmath}
\usepackage{amssymb}
\usepackage{amsthm}
\usepackage{xcolor}
\usepackage{color}
\usepackage{floatrow}
\usepackage{makeidx}
\usepackage{float}
\usepackage{mathtools}
\usepackage{commath}
\usepackage{graphicx}
\graphicspath{{figures/}} 

\usepackage{hyperref}  
\usepackage{url}       
\usepackage{booktabs}  
\usepackage{microtype} 
\usepackage{multirow}
\usepackage{makecell}
\usepackage{algorithm}

\usepackage{algpseudocode}
\usepackage{algcompatible}
\usepackage{siunitx}
\usepackage{subcaption}
\usepackage{soul}
\usepackage{comment}

\usepackage{nicefrac}  
\usepackage{xfrac}

\newcommand{\R}{\mathbb{R}}

\DeclareMathOperator*{\argmin}{argmin}
\DeclareMathOperator*{\snake}{snake}
\DeclareMathOperator*{\rmse}{RMSE}
\DeclareMathOperator*{\relu}{ReLU}

\newtheorem{theorem}{Theorem}[section]
\newtheorem{corollary}{Corollary}[theorem]
\newtheorem{lemma}[theorem]{Lemma}

\newtheorem{remark}[theorem]{Remark}
\theoremstyle{definition}
\newtheorem{definition}{Definition}[section]

\begin{document}

\title[Extrapolation with NN]{Function Extrapolation with Neural Networks and Its Application for Manifolds}


\author{\fnm{Guy} \sur{Hay}} \email{guyhay@mail.tau.ac.il}

\author{\fnm{Nir} \sur{Sharon}}\email{nir.sharon@math.tau.ac.il}


\affil{\orgdiv{School of Mathematical Sciences}, \orgname{Tel Aviv University}, \orgaddress{\street{Ramat Aviv}}, \city{Tel Aviv}, \postcode{6997801}, \country{Israel}}



\abstract{This paper addresses the problem of accurately estimating a function on one domain when only its discrete samples are available on another domain. To answer this challenge, we utilize a neural network, which we train to incorporate prior knowledge of the function. In addition, by carefully analyzing the problem, we obtain a bound on the error over the extrapolation domain and define a condition number for this problem that quantifies the level of difficulty of the setup. Compared to other machine learning methods that provide time series prediction, such as transformers, our approach is suitable for setups where the interpolation and extrapolation regions are general subdomains and, in particular, manifolds. In addition, our construction leads to an improved loss function that helps us boost the accuracy and robustness of our neural network. We conduct comprehensive numerical tests and comparisons of our extrapolation versus standard methods. The results illustrate the effectiveness of our approach in various scenarios.}


\maketitle

\section{Introduction}

Function extrapolation is the problem of estimating the values of a function on one domain from data given at a different domain, where each domain is possibly a manifold. This problem is typically tied to many tasks from various data analysis fields, such as prediction, forecasting~\cite{kimura2002numerical}, convergence acceleration~\cite{scieur2016regularized, han2023riemannian}, extension~\cite{coifman2006geometric}, and continuation. Specifically, function continuation is classic in mathematics; it goes back to Whitney’s extension of smooth functions~\cite{whitney1992functions} when considering continuous functions or even further back if we consider the Prony method~\cite{de1795essai} for sequences and its more modern application for signal processing~\cite{osborne1995modified}. The rise of super-resolution~\cite{candes2014towards} opened the door for exciting new results concerning constructive methods of extrapolation, e.g.,~\cite{demanet2019stable} and more theoretical studies, within the scope of super-resolution~\cite{batenkov2018stable} and inverse problems~\cite{grabovsky2020explicit}, and others~\cite{grabovsky2021optimal, trefethen2020quantifying}.

When it comes to extrapolating a function from its samples over one domain to values required on the other, a key question is what we know (or assume) about it. We roughly divide the answer into two possible options. The first is when only the regularity of the function or its domain is known, e.g.,~\cite{amir2022meshfree}. In such a case, we expect the error to grow exponentially outside the given domain, and so the extrapolation is mostly effective near the boundaries of the function's domain where the data is given, e.g.,~\cite{batenkov2018stable}. A second option is to assume that the function arises from a model. In this case, ideal information, for example, is that the function satisfies a differential equation in the given domain and could be extended beyond it. However, since we have only a discrete set of samples, knowing the differential equation may serve only as a starting point for approximating the function over the domain and extending it~\cite{levin2014behavior}. For this case, standard practice assumes the function lies in or close to a specific known space. In this work, we tackle the latter option, where prior information is used for constructing the extrapolation.

In contrast to time series modeling~\cite{kirchgassner2012introduction}, where the focus is primarily on understanding and predicting the behavior of a variable over time, extrapolation of functions from samples involves extending the understanding of a function beyond the observed data points. Time series modeling typically relies on historical patterns and trends within the data to make future predictions, whereas extrapolation of functions incorporates prior knowledge about the underlying structure or behavior of the function itself, such as, in our case, differential equations or known properties of specific function spaces such as the specific manifold the functions lay in. Therefore, time series modeling and extrapolation are related, and both share the common goal of making informed predictions or estimates based on available information but differ on the underlying assumptions and information they rely on. Our research will focus specifically on function extrapolation.

In recent years, there has been a rising demand for better extrapolation methods, driven by the increasing number of applications and the expanding volume of available data. More and more attention has been paid to creating machine learning algorithms to improve extrapolation capabilities. Unlike approximation, where Neural Networks are known for their remarkable ability, see, e.g.,~\cite{lecun2015deep}, similar promising results have yet to appear in extrapolation. In fact, recent papers show the neural networks' inability to extrapolate well~\cite{ziyin2020neural, belcak2022periodic, parascandolo2016taming}. Specifically, \cite{ziyin2020neural} proves that Neural Networks, which use the common activation functions Tanh and ReLU, lead to constant and linear extrapolations, respectively. Therefore, it proposes a new activation function named Snake, which helps mitigate some of the problems. Other works like EQL~\cite{martius2016extrapolation} and its later improvement~\cite{sahoo2018learning} try to overcome the network extrapolation problem by constructing a network that finds a unique expression for the extrapolated function. EQL's network uses predefined activation functions and outputs an equation by finding a simplified expression as activation functions multiplied by their corresponding coefficients. This is similar to the standard Least Squares (LS) approach when known basis functions are known.

Methods such as Snake, EQL, and LS represent modern and standard extrapolation techniques that involve introducing novel activation functions or approximating the extrapolating function by approximating the available data, also called training data. Such methods do not minimize the extrapolation error since it is unavailable directly from the data. Instead, they aim to fit the training data according to some prescribed metric. It turns out that this approach is sub-optimal, as seen by a connection we draw between the error within the data area and the extrapolation error. In particular, we prove an upper bound on the extrapolation error in terms of the function space, that is, the prior information on the function and the error over the data. Then, we conclude in our analysis that in most typical cases, the error over the data alone does not provide us with the right metric, and an adjustment for the data fitting is required based on the prior information and the extrapolation domain.

Our approach involves using a neural network for extrapolation. The neural network learns the prior information of the function and uses the function samples as the training data to generate a projection onto the learned space. This projection is designed to minimize the extrapolation error. The solution obtained through this approach leverages the full power of neural networks as it addresses approximation rather than direct extrapolation. Interestingly, since our neural network loss function targets the extrapolation error, it potentially overlooks the traditional approximation error. This error alone does not play any significant role in the extrapolation problem despite being central to many conventional methods.

This paper presents a framework for extrapolating functions from their samples. It begins by formally defining the extrapolation problem and dividing the prior information into two cases. The first case assumes that the unknown function is in the span of known basis functions, and the sample is possibly noisy. In the second case, we introduce the concept of anchor functions. These functions are assumed to be known and to lie within a fixed prescribed distance from the unknown function. In this case, we do not have any explicit connection between the anchors and the unknown function other than their proximity. Moreover, in this anchor-extrapolation problem, we assume the unknown functions to exhibit behavior similar to the anchors. In our analysis of the first problem, we prove a bound over the extrapolation error, using the error over the data. This bound introduces a constant we term the condition number of the extrapolation problem. This number quantifies the difficulty level of a given problem with respect to the known basis functions and the two domains in question. 

To illustrate the effectiveness of our Neural Extrapolation Technique, which we term NExT, we apply it to various extrapolation problems. Initially, we extrapolate noisy Chebyshev polynomials of the 3rd, 5th, and 7th degrees. Subsequently, we extend our numerical examples to more intricate cases, such as extrapolating noisy monotone Chebyshev polynomials, which exhibit underlying attributes that are not present in the standard basis. Remarkably, NExT adeptly extrapolates these functions while learning their mathematical properties. Our third challenge is the anchor extrapolation with anchor functions acting as priors. The examples show how NExT outperforms several baseline models, effectively utilizing the capabilities of neural networks. Note that this problem, featuring anchor functions, holds significant potential for real-world applications, such as predicting product sales based on related products, highlighting the versatility of our approach. As our last test case, we consider functions on the sphere in $\mathbb{R}^3$ using the Fourier basis consisting of spherical harmonics, highlighting the benefit when an underlying manifold domain is present. In some examples, we design the settings so the extrapolation domain is not necessarily contiguous with the data domain. This broader testing approach further underlines the robustness and adaptability of  NExT across diverse extrapolation scenarios.

The paper is summarized as follows. Section~\ref{sec:problem_formulation} presents the problem formulation and previous related work in the field. In Section~\ref{sec:extrapolation method}, we introduce our method for extrapolation, including its theoretical derivation. Finally, Section~\ref{sec:ExperimentalResults} provides various numerical examples illustrating the performances of our method over different domains. 

\section{Problem formulation and prior work} \label{sec:problem_formulation}

We consider the problem of extrapolating a function $f$ from its samples and open the discussion with some required notation followed by formulating the problem. Next, we provide a brief intro to neural networks and close this section by recalling some of the prior work done on Least Squares (LS) and deep learning approaches to extrapolation. 

\subsection{Problem formulation}

Let $\mathcal{F}$ be a real vector space defined over $\Omega \cup \Xi$, each can be a different manifold. The domain consists of the samples domain $\Omega$ and the extrapolation domain $\Xi$, both embedded in some Euclidean space $\R^d$. In addition, we assume that $\norm{\cdot}_\Xi$ is a norm induced by an inner product $\langle \cdot, \cdot \rangle_\Xi$, on the extrapolation domain. Then, as a general goal, we look for the following minimizer:
\begin{equation}
g^{\ast}=\arg\min_{g\in \mathcal{F}} \norm{g-f}^2_\Xi .
\label{eqn:ie_ideal_objective}
\end{equation}
We further assume that we are given a basis $\{ \phi_k \}_{k=1}^d$ of $\mathcal{F}$ such that we can represent the unknown function as $f=\sum_{k=1}^d \alpha_k \phi_k$ where $\alpha_k$, $k=1,\ldots,d$ are the unknown basis coefficients. In addition, we have given samples of the form $\{(x_i,y_i)\}_{i=0}^N$ where  $x_i\in \Omega$ and $y_i=f(x_i)$. Since $g^{\ast}$ is uniquely represented by its coefficients $\{ \alpha_{k}^\ast \}_{k=1}^d $, the above extrapolation problem is equivalent to finding these coefficients.

\begin{remark}[More general function spaces]
\label{remark:function_space}
We assume $\mathcal{F}$ is a vector space in the above formulation. However, in some cases, we must facilitate additional prior information. For example, consider function properties like monotonicity or convexity. These properties defined a more restrictive structure of function space within $\mathcal{F}$, which we later show how to incorporate in our learning method.    
\end{remark}

As stated in Remark \ref{remark:function_space}, our formulation can be generalized to function spaces given constraints and the coefficients, therefore we will use the terms 'function space' and 'vector space' interchangeably throughout this paper, unless explicitly stated otherwise. We next introduce two specific instances of the above extrapolation problem, which we address in this paper. In the first version of the extrapolation problem~\eqref{eqn:ie_ideal_objective}, we assume a general, standard model for noisy samples:
\begin{definition}[Extrapolation problem] \label{def_extra_function_problem}
Let $f$ be a real-valued function defined over $\Omega$ and $\Xi$. We observe the data:
\begin{equation} \label{eqn:noisy_samples}
    y_i=f(x_i)+\varepsilon_i, \quad \varepsilon_i \sim \mathcal{N}(0, \sigma^2),\quad x_i \in \Omega \quad i=0,\ldots,N .
\end{equation}
Here, $\sigma$ is unknown. The problem is to estimate $f$ over $\Xi$.
\end{definition}
It is worth noting that this paper focuses on cases of low noise levels. The next problem we consider is extrapolating from samples of a function. However, we do not assume any knowledge of the function being in a certain function space. Instead, we are given anchor functions, which are assumed to be in proximity to the values of the functions over the extrapolation area. We define anchor functions next:
\begin{definition}[Anchor functions]
A set of functions $\{\hat{f}_j\}_{j=1}^M$ are called $\delta$-anchor functions with respect to $f$, if for a given $\delta>0$: 
\begin{equation}
\norm{\hat{f}_j-f}^2_\Xi \leq \delta , \quad j=1,\ldots, M .
\nonumber
\end{equation}
\label{def_anchor_functions}
\end{definition}

Then, the problem reads:
\begin{definition}[Anchored extrapolation problem] \label{def_anchor_function_problem}
Extrapolate a function $f$ from noisy samples of~\eqref{eqn:noisy_samples} to $\Xi$ given a set of $\delta$-anchor functions for a fixed known $\delta >0$.
\end{definition}
In light of the basic problem~\eqref{eqn:ie_ideal_objective}, the anchored problem is formulated to find the function $g^{\ast}$ to $f$ when there is no harsh restriction on the search space other than the distance to the anchors.

\subsection{Neural networks}

Originally inspired by the structure and functioning of the human brain \cite{mcculloch1943logical, lecun2015deep}, neural networks are computational models consisting of nodes, layers, and connections between them. Each layer consists of a predetermined number of nodes, and the whole network consists of a set number of layers. In a feed-forward network, each node is called a perceptron, takes multiple units as inputs, multiplies each input by its corresponding weight, and sums them up. The result is then passed through an activation function. Each node's inputs in a layer $l$ are all node outputs in the previous layer $l-1$ and their weights. We note that the weights for each receiving node are different. The final structure of the neural network is a set of layers of nodes connected to each proceeding and previous layers of nodes by their corresponding weights. A feed-forward neural network is given in~\eqref{eq:neural_network}. Specifically, $X$ is the input vector, $W^{(l)}$ is the weight matrix for layer $l$ and $b^{(l)}$ is its bias vector. The weighted sum $Z^{(l)}$ at layer $l$ leads to $A^{(l)}$, the output of layer $l$ after applying the activation function $a^{(l)}$. The final output is $Y$:
\begin{equation}
\begin{aligned}
    Z^{(1)} &= X \cdot W^{(1)} + b^{(1)} \\
    A^{(1)} &= a^{(1)}(Z^{(1)}) \\
    Z^{(l)} &= A^{(l-1)} \cdot W^{(l)} + b^{(l)} \quad \text{for } l = 2, 3, \ldots, L-1 \\
    A^{(l)} &= a^{(l)}(Z^{(l)}) \quad \text{for } l = 2, 3, \ldots, L-1 \\
    Z^{(L)} &= A^{(L-1)} \cdot W^{(L)} + b^{(L)} \\
    Y &= A^{(L)} = a^{(L)}(Z^{(L)}) .
\end{aligned}
\label{eq:neural_network}
\end{equation}
Here, $L$ is the total number of layers in the neural network. There are many different activation functions, each with its own properties. Two common activation function \cite{ziyin2020neural, apicella2021survey} are ReLU, 
\begin{equation}
    \relu(x) = \max(0, x) ,
    \label{eq:relu}
\end{equation}
and the hyperbolic tangent,
\begin{equation}
    \tanh(x) = \frac{e^{2x} - 1}{e^{2x} + 1} .
    \label{eq:tanh}
\end{equation}

Feedforward neural networks have shown remarkable results, especially in areas with a substantial amount of data \cite{wei2022emergent}. Attesting to their remarkable ability to learn complex structures. One of the fundamental properties that contribute to their widespread applicability \cite{abiodun2018state} is the universal approximation theorem \cite{schafer2006recurrent, cybenko1989approximation}, which asserts that neural networks with a single hidden layer containing a sufficient number of neurons can approximate any continuous function to arbitrary precision. This theorem underscores the versatility and power of neural networks as universal function approximators, allowing them to capture intricate relationships within data.

In light of the above, and with regard to the complexity needed to learn a function space, it becomes evident that neural networks stand out as a natural choice for our model. Our capacity to generate diverse functions within the space ensures a perpetual influx of relevant functions for the model to learn from. Consequently, this capability empowers us to leverage the profound learning potential of neural networks fully. Learning complex mappings from an enormous amount of data.

\subsection{The standard least squares approach}

Least Squares (LS) is a popular extrapolation method \cite{giakas1997improved, garbey2003least}. It consists of finding the optimal coefficients to a predetermined basis while minimizing the squared errors on the training data, solving
\begin{equation}
    \argmin_{\{\alpha_k\}_{k=1}^d} = \sum_{i=1}^N\left(y_i-\sum_{k=1}^d\alpha_k \phi_k(x_i)\right)^2 .
    \label{eq_least_squares}
\end{equation}
Where $\{\alpha_k\}_{k=1}^d$ are the coefficients predicted, $\Phi = \{\phi_k\}_{k=1}^d$ are the basis functions each corresponding to a coefficient, $\{(x_i,y_i)\}_{i=1}^N$ training data consisting of $N$ points. Determining the coefficients allows finding the function that best fits the training data in the LS sense. This function can then be used to predict the extrapolation area.

Motivated by physics, where many equations are differential equations, Prony's method \cite{de1795essai} proposes an approximating function of the sum of complex exponentials with constant coefficients. Given the order beforehand, Prony's method defines a method to compute both the exponential terms and their respective coefficients. An extension of Prony's method to approximate functions with sums of complex exponential is proposed by \cite{levin2014behavior}. The model coefficients are computed using LS on a difference equation for the given points using previous points. Then, the exponents are found as the roots of the characteristic polynomial. Numeric stability with noise is achieved with regularization terms on the model coefficients. Previous works not only proposed new methods for extrapolation but also found important bounds for the error rate. In \cite{grabovsky2020explicit} the authors prove effective bounds of analytical continuations which are exponential and in \cite{batenkov2018stable} prove exponential error bounds on entire functions for stable soft extrapolations, while using a LS polynomial approximation. Both papers attest to the difficulty in accurately extrapolating function given reasonable assumptions. Additional extrapolation methods are given in \cite{brezinski2020extrapolation}, with the notable Richard's extrapolation and Aitken's process. Both important extrapolation methods are used to extrapolate the next element of a sequence.

\subsection{Deep learning approaches for function extrapolation}
\label{sec:deep_learning_approaches}
In recent years, deep learning has seen a few advancements in its capabilities for extrapolation. Extrapolation methods for time series include Prophet~\cite{chen2009overview} and Transformers~\cite{wen2022transformers}. Different from time series, which relay on finding historical patterns and trends, function extrapolation often relay on prior knowledge such as the manifolds the image and domain of the functions lay in, or known structures in the function space. Most notably deep learning advancements are, Snake \cite{ziyin2020neural} and EQL \cite{martius2016extrapolation}. In \cite{ziyin2020neural}, it is proved that conventional activation functions such as ReLU \eqref{eq:relu} and $\tanh$ \eqref{eq:tanh}, produce models that, once $x$ tends to infinity, converge to a linear and constant function, respectively. Thus failing to extrapolate. The authors continue to identify two key components for an activation function that will manage to extrapolate well. The first is monotone, which allows easy convergence since there aren't many local minima. The second is that it should be somewhat periodic so it does not converge to a constant/linear function as x tends to infinity. To this end, they proposed the Snake activation function defined as
\begin{equation}
    \snake(x)=x+\sin^2(\beta x) .
    \label{eq_snake}
\end{equation}
Here, $\beta$ is a learned parameter of the frequency. A second state-of-the-art learner is EQL. EQL stands for Equation Learner; as its name suggests, it learns the underlying equation. It uses a neural network where each activation function in a given layer is different. Thus, the best linear combination of functions with respect to the loss function is chosen. This is the same as LS once mean squared error is used. Once EQL uses multiple hidden layers, it diverges from LS being able to learn to use compound functions. Note that if the best basis elements are known, EQL is equivalent to LS but differs in optimization since the best activation function for this problem will be the basis functions.

\section{Our extrapolation method} \label{sec:extrapolation method}

As stated in Section~\ref{sec:problem_formulation}, each function $g$ can be uniquely identified by a coefficient vector $\alpha$. Therefore, we will neglect the $\alpha$ notation in favor of using $g$ as the coefficient vector as well. Whether $g$ is a function or a coefficient vector will be clear from context, and $g_k$ is the $k$th coefficient of $g$. To extrapolate a function $f\in \mathcal{F}$, NExT will solve \eqref{eqn:ie_ideal_objective} by learning to identify each function $g\in \mathcal{F}$ by its corresponding $y_i$, then it will predict its coefficient vector. Therefore, the neural network $\Upsilon$ associated with NExT is conceived as a mapping $\Upsilon: \mathbb{R}^N \rightarrow \mathbb{R}^d$. Ideally, $\Upsilon$ aims to be a projection onto $\mathcal{F}$ that minimizes the objective function \eqref{eqn:ie_ideal_objective}. Manifold information will be introduced in the basis functions and in the loss function, which will allow strengthening the model.

\subsection{Error analysis}
At the core of our method is our approach for minimizing the extrapolation error by linking it to the data samples and prior information about the function being in the space $\mathcal{F}$. This part deals with exactly that question. Denote the function values over $\Omega$ by $Y_{\Omega} = \{y_i\}_{i=0}^N$, and recall the neural network $\Upsilon$ which aims to predict $g$. We use the following error term to measure the difference between the prediction $\Tilde{g}$ and $g$:
\begin{alignat}{2}
    E_\Xi(\Tilde{g}, g) & = \norm{\sum_{k=1}^d \Upsilon(Y_{\Omega})_k \phi_k - g }^2_\Xi \nonumber \\ & = \norm{\sum_{k=1}^d (\Upsilon(Y_{\Omega})_k - g_k) \phi_k}^2_\Xi \nonumber \\ & = \langle \sum_{k=1}^d (\Upsilon(Y_{\Omega})_k - g_k) \phi_k, \sum_{j=1}^d (\Upsilon(Y_{\Omega})_j - g_j) \phi_j \rangle_\Xi \nonumber \\ & = \sum_{k=1}^d \sum_{\substack{j=1 \\ j\neq k}}^{d} (\Upsilon(Y_{\Omega})_k - g_k)(\Upsilon(Y_{\Omega})_j - g_j)\langle \phi_k,\phi_j\rangle_\Xi + \sum_{k=1}^d (\Upsilon(Y_{\Omega})_k - g_k)^2\norm{\phi_k}^2_\Xi .
    \label{eq_ie_mse_objective_2}
\end{alignat}
Interestingly, the factors of the form $\Upsilon(Y_{\Omega})_{\ell} - g_{\ell}$ are independent of $\Xi$. Therefore, we can interpret \eqref{eq_ie_mse_objective_2} as an error expression where $\| \phi_k \|_\Xi^2$ and $\langle \phi_k, \phi_j \rangle_\Xi$ serving as $\Xi$-depending weights. Therefore, high values of either the inner product or norm of the basis function over $\Xi$ correspond to potentially large errors. 

Denote by $\langle \cdot, \cdot \rangle_\Omega$ an inner product on $\Omega$, and by $\norm{\cdot}_\Omega$ its induced norm. Then, we can consider the expression in~\eqref{eq_ie_mse_objective_2} in terms of $\Omega$, highlighting the differences of minimizing the error over $\Omega$ instead of $\Xi$. Specifically, minimizing over $\Omega$ causes all the factors of~\eqref{eq_ie_mse_objective_2} to be dependent solely on $\Omega$. Therefore, if one uses $E_\Omega(\Tilde{g}, g)$ as a loss function for extrapolation over $\Xi$, as done in many previous methods, it may result in a sub-optimal extrapolation. We summarize it in the following remark:
\begin{remark}
    Previous methods, that do not use learning, are forced to predict extrapolation function while minimizing square errors-like loss functions on the approximation area $\Omega$. Therefore, such methods can assign irrelevant weights, with respect to~\eqref{eq_ie_mse_objective_2}, that is, without directly aiming for minimizing the error over the extrapolation area $\Xi$.
    \label{remark_extra_has_different_coef_weights}
\end{remark}

In the remainder of this subsection, we analyze further the relation between the error over $\Xi$, $E_\Xi$, and the error over $\Omega$, $E_\Omega$. The following remark shows that when we remove the linear independence assumption, made in Theorem~\ref{theorem_1}, over the functions in $\{\phi_k\}^d_{k=1}$, no bounds can be obtained from the approximation error.
\begin{remark} \label{remark_no_assumptions_lead_to_not_knowing}
We present here a simple example showing that the linear independence assumption over $\{\phi_k\}^d_{k=1}$ is essential; In particular, if we are given $\phi_1$ and $\phi_2$ such that $\phi_2|_{\Omega}=-\phi_1$ and $\phi_2|_\Xi=0$. Then, extrapolate $g= \begin{bmatrix} c_1 \\ c_1 \end{bmatrix}$ with a prediction of the form $\Tilde{g}= \begin{bmatrix} c_2 \\ c_2 \end{bmatrix}$ where $c_2\neq c_1$ results in $E_\Omega(g,\Tilde{g})=0$. On the other hand, $E_\Xi(g,\Tilde{g})=(c_2-c_1)\|\phi_1\|^2$, which can be sufficiently large. We note that this problem arises once one of the elements is a linear combination of the rest.
\end{remark}

The observation made in Remark~\ref{remark_no_assumptions_lead_to_not_knowing} implies that when it comes to bound $E_\Xi$ using $E_\Omega$, it becomes necessary to introduce assumptions concerning the basis $\Phi$. Specifically, we consider $\Phi$ as an orthogonal basis and introduce a condition number for extrapolation---a positive number that quantifies how difficult it is to extrapolate from $\Omega$ to $\Xi$ using the orthogonal basis $\Phi$. This condition number measures a worst-case scenario and relies on $\Omega$, $\Xi$, and $\Phi$. Intuitively, the condition consists of two main factors. First is the dimension of the function space since a broader space entails the harder task of determining the minimizer in~\eqref{eqn:ie_ideal_objective}. For example, a single basis function makes the problem quite simple, as one should basically set a single number. The second factor is the scaling of the norms of basis functions, which is also affected by the size of $\Omega$ relative to $\Xi$. The definition reads:  

\begin{definition}[Extrapolation condition number]
 Let $\{\phi_k\}_{k=1}^d$ be an orthogonal basis of $\Omega$. The extrapolation condition number of problem~\eqref{eqn:ie_ideal_objective} is
 \begin{equation} \label{eq:least_squares_extrapolation_condition_number}
 \kappa=d\frac{M_\Xi}{m_\Omega} ,
 \end{equation}
 Where 
 \[ M_\Xi= \max_{k=1,\ldots, d}\norm{\phi_{k}}_\Xi^2, \quad  \text{ and } \quad m_\Omega= \min_{k=1,\ldots, d}\norm{\phi_{k} }_\Omega^2. \]
\label{def_least_squares_extrapolation_condition_number}
\end{definition}

The above Extrapolation Condition Number will play a role in our analysis when bounding the extrapolation error in Theorem~\ref{theorem_1}. In addition, in the experimental section, we demonstrate its applicability over numerical examples.
\begin{remark}Two remarks regarding Definition~\ref{def_least_squares_extrapolation_condition_number}:
\begin{enumerate}
    \item 
    The condition number~\eqref{eq:least_squares_extrapolation_condition_number} is, by definition, a positive value. While a higher value indicates more uncertainty in determining the extrapolation, it is worth noting that it is possible to obtain smaller values, even less than $1$. As in other functional condition numbers, a smaller value may occur. Specifically, if $M_\Xi < m_\Omega$, arising from either an irregular basis where the basis elements are larger in scale and norm over $\Xi$ or when $\Xi$ is significantly smaller than $\Omega$. This means the extrapolation problem quantitatively belongs to a class of easier settings.
    \item 
    We assume the orthogonality of our basis over $\Omega$, which tightly relates the condition number~\eqref{eq:least_squares_extrapolation_condition_number} to the extrapolation error, as we will see next. However, in practice, we can use this measure of hardness to assess extrapolation even if the basis is not orthogonal. This practice is also demonstrated in Section~\ref{sec:ExperimentalResults}.
\end{enumerate}
\end{remark}

\begin{theorem}
    Let $\{\phi_k\}_{k=1}^d$ be a sequence of real-valued functions, defined both in $\Omega$ and $\Xi$. The following conditions hold:
    \begin{enumerate}
        \item If $\phi_k$ are orthogonal in $\Omega$, then $E_\Xi \leq \kappa E_\Omega$.
        \item If $\phi_k$ are orthogonal both in $\Omega$ and $\Xi$, then $E_\Xi \leq \frac{\kappa}{d}E_\Omega$.
    \end{enumerate}
    Where $M_\Xi$, and $m_\Omega$ are defined in Definition \ref{def_least_squares_extrapolation_condition_number}, and $E_\Omega$ is the error in \eqref{eq_ie_mse_objective_2} for $\Omega$ instead of $\Xi$.
    \label{theorem_1}
\end{theorem}
We first introduce a lemma that serves as an auxiliary tool in the proof of Theorem~\ref{theorem_1}.
\begin{lemma} \label{lemma_1}
    Let $a=\left(a_1,\ldots,a_d\right)$ be a vector of positive numbers. Then,
    \begin{equation} \label{eqn:lemma_1}
        \dfrac{\sum_{k=1}^d \sum_{\substack{j=1 \\ j\neq k}}^{d} a_k a_j + \sum_{k=1}^da_k^2}{\sum_{k=1}^da_k^2} \leq d .
    \end{equation}
\end{lemma}
\begin{proof}
    Since $\sum_{k=1}^da_k= \norm{a}_{L_1}$ we have that
\[ \sum_{k=1}^d \sum_{\substack{j=1 \\ j\neq k}}^{d} a_k a_j + \sum_{k=1}^da_k^2 = \sum_{k=1}^d a_k (\norm{a}_{L_1}-a_k) + \sum_{k=1}^da_k^2 = (\norm{a}_{L_1})\sum_{k=1}^d a_k =\norm{a}_{L_1}^2. \]
Therefore, the left hand side of~\eqref{eqn:lemma_1} is bounded by $\tfrac{\langle 1,a \rangle}{\norm{a}_{L_2} ^2}$ which in turn is bounded, using cauchy-schwarz, by $ d$.
\end{proof}
\begin{proof}[Proof of Theorem~\ref{theorem_1}]
    Proof of Condition 1:
    We assume $E_\Omega \neq 0$ otherwise we get from \eqref{eq_ie_mse_objective_2} that $\Tilde{g}=g$ since from orthogonality at $\Omega$ we get that $\langle \phi_k,\phi_j\rangle_\Omega = 0$ thus $0=E_\Omega=\sum_{k=1}^d (\Upsilon(Y_{\Omega})_k - g_k)^2 \norm{\phi_k}_\Omega^2 \rightarrow \forall (\Upsilon(Y_{\Omega})_k - g_k)=0 \rightarrow g=\Tilde{g}$ and we are done. Assuming $E_\Omega \neq 0$ and $\{\phi_k\}_{k=1}^d$ are orthogonal at $\Omega$, as mentioned $\langle \phi_k,\phi_j\rangle_\Omega = 0$, defining $\phi_{\max_\Xi}=\arg \max_{\phi_k}\norm{\phi_{k}}_\Xi^2$ we get that $\langle \phi_k,\phi_j\rangle_\Xi \leq M_\Xi$ and $\forall k \quad \norm{ \phi_{k}}_\Xi \leq \norm{\phi_{\max_\Xi}}_\Xi$. Similarly from definition we get $\forall k \quad \norm{\phi_{k}}_\Omega \geq \norm{\phi_{\min_\Omega}}_\Omega$. Using the above with \eqref{eq_ie_mse_objective_2} then:

   \begin{alignat}{2}
    \frac{E_\Xi(\Tilde{g}, g)}{E_\Omega(\Tilde{g}, g)} & = \frac{\sum_{k=1}^d \sum_{\substack{j=1 \\ j\neq k}}^{d} (\Upsilon(Y_{\Omega})_k - g_k)(\Upsilon(Y_{\Omega})_j - g_j)\langle \phi_k,\phi_j\rangle_\Xi}{\sum_{k=1}^d (\Upsilon(Y_{\Omega})_k - g_k)^2\norm{\phi_k}^2_\Omega} + \nonumber \\
    & \frac{\sum_{k=1}^d (\Upsilon(Y_{\Omega})_k - g_k)^2\norm{\phi_k}^2_\Xi}{\sum_{k=1}^d (\Upsilon(Y_{\Omega})_k - g_k)^2\norm{\phi_k}^2_\Omega} \nonumber \\
    & \le \frac{\sum_{k=1}^d \sum_{\substack{j=1 \\ j\neq k}}^{d} (\Upsilon(Y_{\Omega})_k - g_k)(\Upsilon(Y_{\Omega})_j - g_j)\norm{ \phi_k}_\Xi\norm{\phi_j}_\Xi}{\sum_{k=1}^d (\Upsilon(Y_{\Omega})_k - g_k)^2\norm{\phi_k}^2_\Omega} + \nonumber \\
    & \frac{\sum_{k=1}^d (\Upsilon(Y_{\Omega})_k - g_k)^2\norm{\phi_k}^2_\Xi}{\sum_{k=1}^d (\Upsilon(Y_{\Omega})_k - g_k)^2\norm{\phi_k}^2_\Omega} \nonumber \\
    & \leq \frac{M_\Xi}{m_\Omega} \frac{\sum_{k=1}^d \sum_{\substack{j=1 \\ j\neq k}}^{d} |\Upsilon(Y_{\Omega})_k - g_k||\Upsilon(Y_{\Omega})_j - g_j|   + \sum_{k=1}^d (\Upsilon(Y_{\Omega})_k - g_k)^2 }{\sum_{k=1}^d (\Upsilon(Y_{\Omega})_k - g_k)^2 } \nonumber \\
    & \leq d\frac{M_\Xi}{m_\Omega}. \nonumber
\end{alignat}

Where the last inequality is from Lemma \ref{lemma_1} while defining $a_k=|\Upsilon(Y_{\Omega})_k - g_k|$.

Proof of Condition 2:
    Similarly to Proof of Condition 1, we assume $E_\Omega \neq 0$. Assuming orthogonality in $\Omega$ and $\Xi$ leads to $\langle \phi_k,\phi_j\rangle_\Xi=0$ and $\langle \phi_k,\phi_j\rangle_\Omega=0$ for $k\neq j$. Thus:
    \begin{alignat}{2}
    \frac{E_\Xi(\Tilde{g}, g)}{E_\Omega(\Tilde{g}, g)} & \leq \frac{\sum_{k=1}^d (\Upsilon(Y_{\Omega})_k - g_k)^2\|\phi_k\|^2_\Xi}{\sum_{k=1}^d (\Upsilon(Y_{\Omega})_k - g_k)^2\norm{\phi_k}^2_\Omega}  \nonumber \leq \frac{M_\Xi}{m_\Omega}.
\end{alignat}
\end{proof}

The initial inequality in Theorem \ref{theorem_1} underscores the limitation of prior extrapolation methods focused on $\Omega$, which aim to minimize $E_\Omega$. While such methods can yield favorable outcomes when $\kappa$ is sufficiently small, they become less effective when $\kappa$ is not. Additionally, as highlighted in Remark \ref{remark_no_assumptions_lead_to_not_knowing}, the endeavor to minimize $E_\Omega$ lacks direction when no specific conditions are specified. From Theorem \ref{theorem_1}, we highlight Remark \ref{remark_less_basis_elements}, a somewhat counterintuitive result, which states that using fewer basis elements may produce better results. In addition, through Definition~\ref{def_least_squares_extrapolation_condition_number} of the extrapolation condition number, we can conclude an intuitive result that that in the general case, it is beneficial to have a large $\Omega$ with a small $\Xi$. Given by the denominator and nominator of $\kappa$ \eqref{eq:least_squares_extrapolation_condition_number}.

\begin{remark} [Basis overfitting]
    In light of the above discussion, specifically Definition~\ref{def_least_squares_extrapolation_condition_number} and Theorem~\ref{theorem_1}, we interpret the overfitting phenomenon in extrapolation as follows: having more basis functions can increase the condition number since then $d$ increases, $M_\Xi$ may increase, and  $m_\Omega$ may decrease. Therefore, our bound and its associated condition number show how too many basis functions can cause less favorable results and analytically explain the overfitting problem.
    \label{remark_less_basis_elements}
\end{remark}

\begin{remark} [Examples for part 2 of Theorem~\ref{theorem_1}] \label{remark_theorem_1_2}
To illustrate part 2 of Theorem~\ref{theorem_1}, we consider doubly-orthogonal functions. The doubly orthogonal function system is an orthogonal function system in which functions are orthogonal in two different inner product spaces consisting of different weighting functions. This concept goes back to Bergman, see, e.g.,~\cite{shapiro1979stefan} and generalizes phenomena such as the orthogonality of Chebyshev polynomials over both an interval (with weight function) and an ellipse and the Slepian functions~\cite{wang2017review} which are orthogonal in a closed real segment and over the entire real line. 

Note that this paper focuses on the more general case, namely, orthogonality just on $\Omega$, as in part 1 of Theorem~\ref{theorem_1}.

\end{remark}

\begin{corollary}
When $\{\phi_k\}_{k=1}^d$ are orthonormal in both $\Omega$ and $\Xi$, minimizing $E_\Omega$ bounds $E_\Xi$.
\end{corollary}
\begin{proof}
    By Theorem \ref{theorem_1}, with orthonormal basis on both $\Omega$ and $\Xi$, we get $E_\Xi \leq E_\Omega$
\end{proof}

\subsection{An overview of the NExT framework}
\label{sec:next_overview}

In NExT, the learning process consists of randomly sampling $g\sim \mathcal{F}$. As our default, we sample according to a multivariate normal distribution with zero mean. Whether a better sampling scheme exists is left for further research. Evaluating $g$ on $\{x_i\}_{i=1}^N\subset \Omega$, resulting in $Y_\Omega$, which is one training sample with a multi-regression label of the coefficients of $g$. NExT will have a weighted MSE loss on the coefficients solving~\eqref{eq_ie_mse_objective_2}, symbolized by $\mathcal{L}_{\text{core}}$. When $\mathcal{F}$ is a function space but not a vector space, ideally, $g$ will be sampled from that space. In cases where such sampling is impossible, $g$ can be projected to the space so long as its coefficient representation is still possible. Additionally, in such cases, it is possible to add additional losses on the coefficients and the extrapolation area $\Xi$, which we denote by $\mathcal{L}_{ext}$. This loss term will help include information on $\mathcal{F}$ when available; for example, if $\mathcal{F}$ consists of monotonically increasing functions, it can be enforced by additional loss terms, as mentioned in Remark \ref{remark:function_space}. In such a case, for instance, one can use $\mathcal{L}_{ext}=\relu(g(x'_0)-g(x'_{M'}))$, where $x'_0$ is the beginning of $\Xi$ and $X'_M$ is the end. The overall loss is: 
\begin{alignat}{2}
    \mathcal{L}(g, \Upsilon(Y_{\Omega}) = 
    & \lambda_{ext} \mathcal{L}_{ext}(Y_\Xi, \hat{Y}_{\Xi}) + \nonumber \\
    & \lambda_{\text{core}} \mathcal{L}_{\text{core}}(g, \Upsilon(Y_{\Omega})) .
\label{eq_ie_loss_function}
\end{alignat}
Where $\hat{Y}_{\Xi}=\{\Upsilon(Y_\Omega)(x_i)\}_{i=1}^N$ with respect to the sample $\{x_i\}_{i=1}^N\subset \Xi$. Additionally, the ability of the neural network to converge relies heavily on the search space size. Therefore, NExT introduces four additional hyper-parameters $r_m,r_\sigma \in \mathbb{R}$ and $ N_l,N_h\in \mathbb{N}$ which bounds $\mathcal{F}$. The parameters $r_m,r_\sigma$ are the normalization range of the functions coefficients meaning $\norm{g}_{L_2} \sim \mathcal{N}(r_m, r_\sigma^2)$. The other two parameters $N_l, N_h$ determine the basis functions for the learning, $0\leq N_l\leq N \leq N_h \leq d$. Namely, once setting the list of basis function, $\phi_1,\ldots,\phi_d$, the randomly generated functions will be of the form $g=\sum_{k=N_l}^{N_h} g_k \phi_k$. NExT training process is described in Algorithm \ref{alg:training}.

\begin{algorithm}[hb]
\caption{Training NExT}\label{alg:training}
\begin{algorithmic}
\STATE {\bfseries Input:} $\Omega$, $\Xi$, $\{x_i\}_{i=1}^N\subset \Omega$ and basis elements $\{\phi_k\}_{k=1}^d$ of $\mathcal{F}$. \\ \textit{Hyperparameters}: batch size $n_{\text{batch}}$, coefficient normalization mean $r_m$ and standard deviation $r_{\sigma^2}$. The minimum and maximum number of basis functions to use $N_l, N_h$
\STATE {\bfseries Output:} Trained neural network $\Upsilon$.
\WHILE{$\Upsilon$ not converged}
\STATE Sample $n_{\text{batch}}$ of function
\FOR{\texttt{j=1} \textbf{to} $n_{\text{batch}}$}
\STATE Sample $N_h - N_l$ basis coefficients from a multi-normal distribution $g_j\sim \mathcal{N}(0,1)$
\STATE Sample normalization value $\alpha_j \sim \mathcal{N}(r_m,r_{\sigma^2})$
\STATE Normalize $g$ such that $\norm{g_j}_{L_2}=\alpha_j$
\IF{$\mathcal{F}$ is not a vector space}
\STATE Project $g_j$ onto $\mathcal{F}$.
\ENDIF
\STATE Define $Y_\Omega$ as the set of $\{g_j(x_i)\}_i^N$
\STATE Define training sample as $\{Y_\Omega, g_j \}$
\ENDFOR
\STATE Train $\Upsilon$ with generated training samples using loss function $\mathcal{L}$ in \eqref{eq_ie_loss_function}.
\ENDWHILE
\end{algorithmic}
\end{algorithm}

To sum, by training with Algorithm~\ref{alg:training}, NExT learns to approximate the function space $\mathcal{F}$ while focusing on minimizing $E_\Xi$ directly. 

We will continue inspecting the extrapolation error. In particular, assuming again that $\mathcal{F}$ is a vector space with an inner product. We consider the case where $f$ does not necessarily lie in $\mathcal{F}$, as often occurs. In such a case, we denote by $f_\parallel, f_\perp$ the projection of $f$ to $\mathcal{F}$ and the perpendicular part of it, respectively. Therefore $g^{\ast} = f_{\parallel}$ from definition. Then,  
\begin{alignat}{2}
    E_\Xi(f,\Upsilon(Y_{\Omega})) & = \norm{\Upsilon(Y_{\Omega}) - f}^2_\Xi \nonumber  = \norm{\sum_{k=1}^d \Upsilon(Y_{\Omega})_k \phi_k - f}^2_\Xi \nonumber \\ & = \norm{\sum_{k=1}^d (\Upsilon(Y_{\Omega})_k - f_{\parallel_k}) \phi_k + f_\perp}^2_\Xi \nonumber \\ & \leq \sum_{k=1}^d \sum_{\substack{j=1 \\ j\neq k}}^{d} (\Upsilon(Y_{\Omega})_k - f_{\parallel_k})(\Upsilon(Y_{\Omega})_j - f_{\parallel_j})\langle \phi_k,\phi_j \rangle_\Xi \nonumber \\ & + \sum_{k=1}^d (\Upsilon(Y_{\Omega})_k - f_{\parallel_k})^2\norm{\phi_k}^2_\Xi + \norm{f_\perp}_\Xi^2 \nonumber \\
    & = \sum_{k=1}^d \sum_{\substack{j=1 \\ j\neq k}}^{d} (\Upsilon(Y_{\Omega})_k - g^{\ast}_k)(\Upsilon(Y_{\Omega})_j - g^{\ast}_j)\langle \phi_k,\phi_j \rangle_\Xi \nonumber \\ & + \sum_{k=1}^d (\Upsilon(Y_{\Omega})_k - g^{\ast}_k)^2\norm{\phi_k}^2_\Xi + \norm{f_\perp}_\Xi^2 \nonumber \\
    & =  \Big\| \Upsilon\left(Y_{\Omega}\right) - g^{\ast} \Big\|^2_\Xi + \norm{g^{\ast} - f}^2_\Xi \leq \max_{g\in \mathcal{F}}E_\Xi(g,\Upsilon(Y_{\Omega})) + \norm{g^{\ast} - f}^2_\Xi .
    \label{eq_ie_mse_objective_f}
\end{alignat}

Where $\max_{g\in \mathcal{F}}E_\Xi(g,\Upsilon(Y_{\Omega}))$ is the maximum error $\Upsilon$ has over $\mathcal{F}$ and $g^{\ast}$ defined in \eqref{eqn:ie_ideal_objective}. Therefore, \eqref{eq_ie_mse_objective_f} implies that two terms bound the extrapolation error: first is the error between the network prediction over functions in the space. The second is how far the space is from the target function $f$.

It is reasonable to assume that the sampling scheme used for the sampling of the function can help improve the ability of the network to extrapolate functions from the learned function space. Therefore, we propose a possible sampling scheme in Remark~\ref{remark_1}:
\begin{remark}[Function sampling for training]\label{remark_1}
The final term in \eqref{eq_ie_mse_objective_f} can be improved once the distribution of sampling $g$ on $\mathcal{F}$ is optimized, as it affects NExT's ability to identify $g^{\ast}$ from $Y_{\Omega}$. Therefore, we ideally want a sampling scheme that samples vectors in the coefficient vector space and lowers the error term $\max_{g\in \mathcal{F}}E_\Xi(g,\Upsilon(Y_{\Omega}))$.

Using~\eqref{eq_ie_mse_objective_f}, we notice that an error in the prediction of each coefficient given by $(\Upsilon(Y_{\Omega})_k - g^{\ast}_k)$, causes a different change to the overall error associated with the prediction of $g^{\ast}$, given by $$\sum_{k=1}^d \sum_{\substack{j=1 \\ j\neq k}}^{d} (\Upsilon(Y_{\Omega})_k - g^{\ast}_k)(\Upsilon(Y_{\Omega})_j - g^{\ast}_j)\langle \phi_k,\phi_j \rangle_\Xi + \sum_{k=1}^d (\Upsilon(Y_{\Omega})_k - g^{\ast}_k)^2\norm{\phi_k}^2_\Xi,$$ because of the different sizes of the inner products between the basis functions. Hence, we would like our sampling scheme to allow $\Upsilon$ to predict the most important coefficients better, where their importance can be calculated by the relative size of the overall sum of the inner products with other basis elements and with themselves. Under the assumption that $\Upsilon$ prediction ability is directly related to the number of instances it has in its training data, we would like to sample in closer margins important coefficients while using larger margins for less important coefficients. This will allow the training data to consist of many samples where the less important coefficient is left unchanged but with a change to the more important one. Thus allowing $\Upsilon$ to predict the more important coefficients better.
\end{remark}

The generality of the NExT framework allows us to consider various data domains $\Omega$ and extrapolation domains $\Xi$. One particular choice is domains over a manifold. One challenge in extrapolating over manifold domains is defining the function spaces and, in particular, fixing prior in terms of basis functions. We assume a compact, connected Riemannian manifold to make our discussion more concrete. Then, a natural basis arises from the Laplace–Beltrami operator, e.g.,~\cite{berard2006spectral}. In Section~\ref{sec:ExperimentalResults}, we demonstrate such settings, using the sphere as an example compact manifold and two of its distinct subdomains as $\Omega$ and $\Xi$.

\subsection{Extrapolating with anchor functions}
\label{sec:anchor_extrapolation}

Here, we discuss the anchored extrapolation problem of Definition~\ref{def_anchor_function_problem}. We propose using the same method described in Section~\ref{sec:next_overview}, but with a few necessary adjustments since here, we are given a set of functions that do not necessarily form a basis, while making sure that no anchor function is a linear combination of the rest, in light of Remark~\ref{remark_no_assumptions_lead_to_not_knowing}. Recall that by definition, each anchor function $\hat{f}_j$ satisfies $E_\Xi(f,\hat{f}_j) \leq \delta$. Namely, all anchor functions are in a ball of radius $\delta$ around $f$. Therefore, when considering their corresponding span, $\mathcal{F}$, any minimizer of~\eqref{eqn:ie_ideal_objective}, $g^{\ast} \in \mathcal{F}$, certainly satisfies $\norm{g^{\ast} - f}_\Xi \leq \delta$. 

Recall the bound~\eqref{eq_ie_mse_objective_f}. In the case of anchor extrapolation, the bound and its conclusions are specifically relevant as $f$ is not guaranteed to be in $\mathcal{F}$. Therefore, we consider the two terms and define an extended space $\mathcal{\Tilde{F}}$ such that $\mathcal{F} \subset \mathcal{\Tilde{F}}$. On the one hand, there exists $\Tilde{g}^* \in \mathcal{\Tilde{F}}$ so
\begin{equation}
    \norm{g^{\ast} - f}_\Xi \geq \norm{\Tilde{g}^* - f}_\Xi .
    \label{eq_better_filler}
\end{equation}
Therefore, minimizing one term in~\eqref{eq_ie_mse_objective_f}. On the other hand, to form $\mathcal{\Tilde{F}}$, we add extra functions, which we term ``filler'' functions. This construction may lead to a higher error in the second term as the filler functions are naturally further away from the data leading to a typical inequality: $\max_{g\in \mathcal{F}}E_\Xi(g,\Upsilon(Y_{\Omega})) \leq  \max_{g\in \mathcal{\Tilde{F}}}E_\Xi(g,\Upsilon(Y_{\Omega}))$. Namely, the added filler functions introduce a tradeoff between the two parts of~\eqref{eq_ie_mse_objective_f}. However, in practice, we show that the benefit of adding filler functions is preferable, as seen over the different test scenarios of Section~\ref{sec:ExperimentalResults}.

As for choosing the filler functions, since the anchor functions are close to the target function $f$, the filler functions should be bounded in $\Xi$ with a small constant. This way, we get more freedom to fit the extrapolated function, which is more likely to allow a better extrapolating function to be found. A specific choice of such function is case-dependent and is illustrated in the next section.

\section{Experimental results} \label{sec:ExperimentalResults}

We assess the performance of our proposed method through experiments conducted on two distinct extrapolation problems: extrapolating with a known basis, as described in Equation~\eqref{eqn:noisy_samples}, and the anchored extrapolation problem outlined in Definition~\ref{def_anchor_function_problem}. Our experimentation involves multiple phases. Initially, we address the first extrapolation problem in two instances: once using a 7-degree Chebyshev polynomials function space and another time focusing on a subset of 7-degree monotonic Chebyshev polynomials functions. Subsequently, we tackle the anchor problem by employing two sets of anchor functions: decaying functions and non-decaying functions. We proceed to demonstrate the relative strength of NExT by solving the anchor function problem in regions further from the known data. Furthermore, we compare the noise sensitivity between LS and NExT in our extrapolation settings. Lastly, we showcase NExT's effectiveness in the manifold setting by extrapolating using spherical harmonics basis functions \cite{schonefeld2005spherical}.

The preferred metric to evaluate NExT and the baseline models is Root Mean Squared Error (RMSE) formulated in \eqref{eq_rmse}. RMSE gives larger error values to larger errors but keeps the result in the original units. Therefore, it is ideal for our evaluation. RMSE is commonly used to evaluate extrapolation; see, e.g.,~\cite{martius2016extrapolation, sahoo2018learning}. Formally, the RMSE reads:
\begin{equation}
    \rmse(\{y_i\}_{i=1}^{N_e},\{\hat{y}_i\}_{i=1}^{N_e}) = \sqrt{\frac{1}{N}\sum_{i=1}^{N_e} (y_i - \hat{y_i})^2} .
    \label{eq_rmse}
\end{equation}
Here $\{y_i\}_{i=1}^{N_e}$ are the true functions values in the extrapolation area and $\{\hat{y}_i\}_{i=1}^{N_e}$ are the predicted ones. So, to evaluate a given model, $N_e$ equally spaced points will be chosen from the extrapolation area and compared to their predicted values. Since the square root function is monotonically increasing, minimizing RMSE is equivalent to minimizing MSE. Thus, our formulation in~\eqref{eq_ie_mse_objective_2} holds true.

Since we wish to show NExT's contribution even in small noise scenarios, we would like to add a small noise factor to the function. The noise signal will be measured by signal-to-noise ratio (SNR) \cite{johnson2006signal}:
\begin{equation}
    SNR = 10\log_{10}(\frac{P_s}{P_n}) .
    \label{eq:snr}
\end{equation}
Where $P_s, P_n$ are the squares of the $L_2$ norm of the signal and noise, respectively.

\subsection{Noisy function samples} \label{sec:global_function_space}
Our experiment focus on extrapolating real-valued functions with $\Omega=[-1,0.5),\Xi=[0.5,1], f: \Omega \rightarrow \Xi$. A natural basis choice would be the Chebyshev polynomials~\cite{mason2002chebyshev}. We use LS implemented by Numpy~\cite{harris2020array} as the baseline model. LS represents the standard method for extrapolation using coefficients for known basis functions, and as stated in Section~\ref{sec:deep_learning_approaches}, EQL is equivalent to it once the basis functions are known. LS uses the basic underlying framework as NExT, which relies on finding the coefficients of the function and, therefore, is an important baseline. For this problem, $\kappa=22974.71$; for reference, the trigonometric basis of sine and cosine functions has a $\kappa=11.03$. We note that since Chebyshev polynomials are only orthogonal at $[-1,1]$, Theorem~\ref{theorem_1} does not hold for $\Omega$. Yet, $\kappa$ still gives insight into how difficult extrapolating with Chebyshev polynomial is for this problem.

For both LS and NExT methods, training and evaluation data have additional noise factors of SNR=35. An SNR value of 35 is considered a low noise level and is a typical error one gets after applying denoising methods on higher noise levels, e.g.,~\cite{dangeti2003denoising, almahamdy2014performance}. 

To fully validate NExT, we first train on functions generated with $r_\sigma=0.25, r_m=1, N_l=0, N_h=7$. The validation set consists of three sets of $100$ randomly generated functions with $r_\sigma=0.25, r_m=1$, and each set has one specific degree: $3,5$, or $7$. The three different degrees are used to evaluate how all methods respond to growing complexities.

\begin{figure}[!htp]
    \centering
    \begin{subfigure}[t]{0.45\textwidth}
        \includegraphics[width=\textwidth]{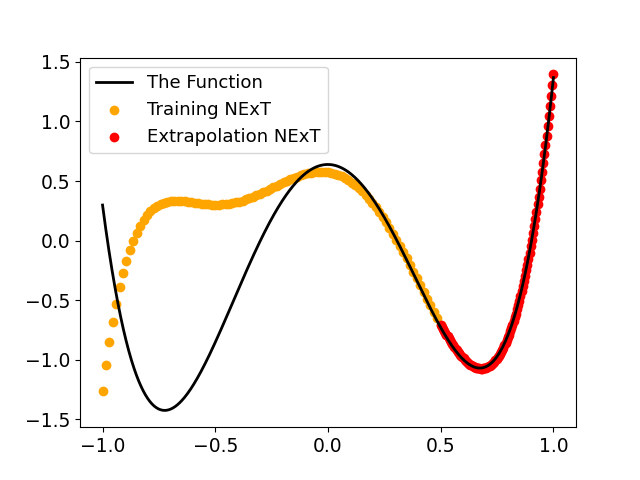}
        \caption*{}    
    \end{subfigure}\qquad
    \begin{subfigure}[t]{0.45\textwidth}
        \includegraphics[width=\textwidth]{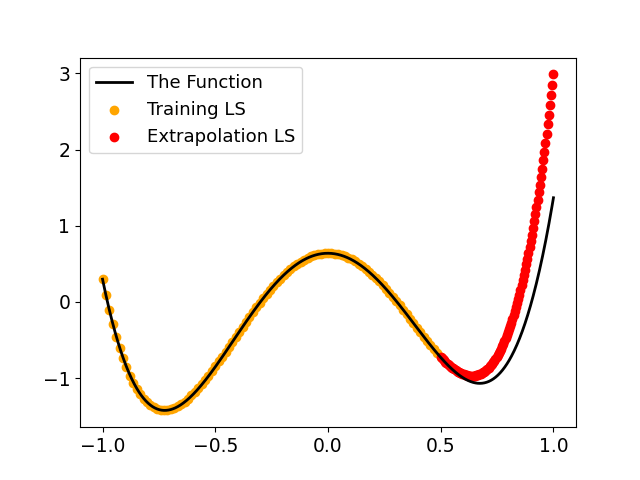}
        \caption*{}  
    \end{subfigure}
    \caption{A comparison between our method (left) and an LS-based extrapolation (right). The space consists of a 5th-degree Chebyshev polynomial where the LS achieves an RMSE score of 0.633, while NExT outperforms it with 0.015, achieving a 97.7\% error reduction rate. The LS model clearly overfits the training data domain $\Omega$, where the data is given, while NExT manages to focus on minimizing~\eqref{eq_ie_mse_objective_2} over $\Xi$.}
\label{fig_LS_vs_ie_5}
\end{figure}

\begin{table*}[htbp]
\centering
\resizebox{1\textwidth}{!}{
\begin{tabular}{@{}ccccccccccc@{}}
\toprule
&   
\multicolumn{3}{c}{NExT} & 
\multicolumn{3}{c}{LS}\\
\cmidrule(lr){2-4}\cmidrule(lr){5-7}
Num coefficients & $\Xi$ RMSE & Coefficients RMSE & $\Omega$ RMSE & $\Xi$ RMSE & Coefficients RMSE & $\Omega$ RMSE \\
\midrule
3 & \textbf{0.021} & 0.387 & 0.598 & 0.263 & 0.155 & 0.003\\
\midrule
5 & \textbf{0.054} & 0.334 & 0.597 & 0.291 & 0.124 & 0.004\\
\midrule
7 & \textbf{0.200} & 0.299 & 0.605 & 0.267 & 0.114 & 0.003 \\
\bottomrule
\end{tabular}
}
\caption{A comparison between NExT and LS. The RMSE over $\Xi$ indicates that the NExT outperforms the LS method. Note the higher values of the coefficients RMSE and $\Omega$ RMSE, which show how NExT manages to focus on the relevant area and truly minimizes~\eqref{eq_ie_mse_objective_2}. While the LS method fails to do so and focuses too heavily on $\Omega$, it loses its extrapolation ability on $\Xi$.}
\label{tbl:experiments}
\end{table*}

The results for 3 different degree polynomials with 35 snr noise added are in Table. \ref{tbl:experiments}. LS achieved RMSE scores of $0.267, 0.291$, and $0.263$ on the $3,5$, and $7$-degree polynomials, whereas NExT achieved scores of 0.021, 0.054, and 0.200, respectively. This amounts to error reduction rates of 92.1\%, 81.4\%, and 31.5\%, respectively. LS RMSE results on the approximation area $\Omega$ are an order of magnitude better than NExTs, showcasing how NExT focuses on the extrapolation area alone, substantially increasing its extrapolation ability.

\subsection{Monotonic functions with noisy samples}

We focus on noisy monotonic functions to evaluate how NExT performs on function spaces with an underlying structure not present in the basis functions. As mentioned in Remark~\ref{remark:function_space}, this is an example of using a function space that is not a vector space. Under the same evaluation method, we learn noisy monotonic functions generated by adding the minimal value of the polynomial to the first coefficient and then integrating the coefficients to get a monotonically increasing function. In the following, NExT trained on monotonic functions will be called NExT Monotonic. While using NExT to learn the whole function space, it is called NExT Whole Space. As in Section~\ref{sec:global_function_space}, a noise of SNR=35 was added.

\begin{figure}[htp]
    \centering
    \begin{subfigure}[t]{0.45\textwidth}
        \includegraphics[width=\textwidth]{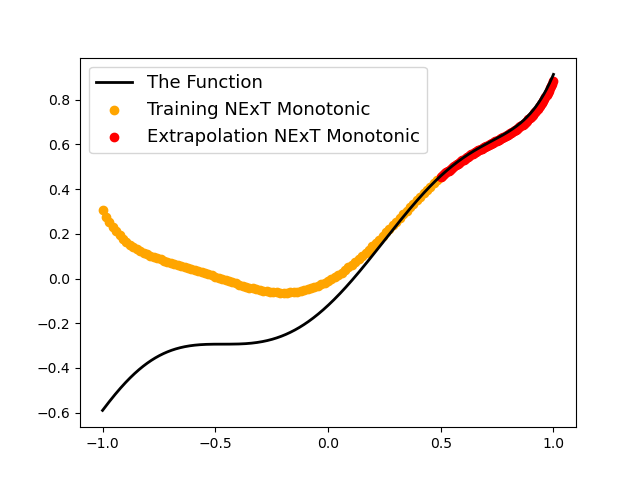}
        \caption*{}    
    \end{subfigure}\qquad
    \begin{subfigure}[t]{0.45\textwidth}
        \includegraphics[width=\textwidth]{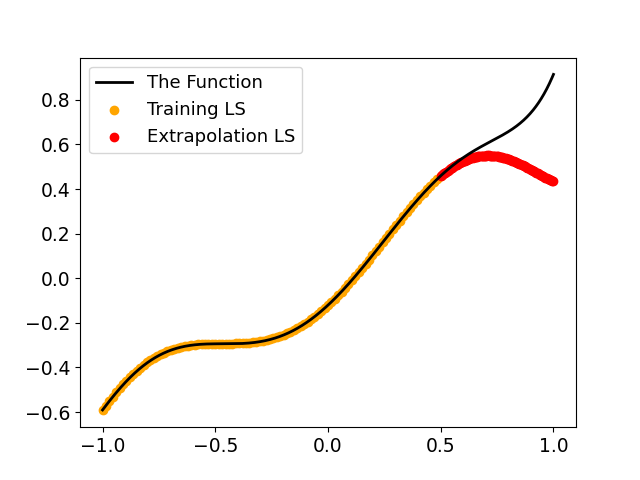}
        \caption*{}  
    \end{subfigure}
\caption{A comparison between our model (left) and an LS-based extrapolation (right) for a specific 7-degree monotonic Chebyshev polynomial. LS achieves an RMSE score of 0.190, while NExT outperforms it with 0.015, achieving a 92.3\% error reduction rate. The LS model clearly overfits the training data domain $\Omega$, where the data is given, while NExT manages to focus on minimizing~\eqref{eq_ie_mse_objective_2} over $\Xi$. In addition, although both didn't predict a true monotonic function, NExT managed to predict a monotonic function in $\Xi$, indicating it learned to predict this kind of functions, at least in the desired domain $\Xi$.}
\label{fig_LS_vs_ie_7_monotonic}
\end{figure}

\begin{table*}[htbp]
\centering
\resizebox{1\textwidth}{!}{
\begin{tabular}{@{}cccccccccccccc@{}}
\toprule
&   
\multicolumn{3}{c}{NExT Monotonic} & 
\multicolumn{3}{c}{NExT Whole Space} & 
\multicolumn{3}{c}{LS} \\
\cmidrule(lr){2-4}\cmidrule(lr){5-7}\cmidrule(lr){8-10}
Num coefficients & $\Xi$ RMSE & Coefficients RMSE & $\Omega$ RMSE & $\Xi$ RMSE & Coefficients RMSE & $\Omega$ RMSE & $\Xi$ RMSE & Coefficients RMSE & $\Omega$ RMSE \\
\midrule
3 & 0.016 & 0.367 & 0.456 & \textbf{0.012} & 0.411 & 0.603 & 0.124 & 0.072 & 0.002 \\
\midrule
5 & \textbf{0.022} & 0.372 & 0.618 & 0.062 & 0.414 & 0.777 & 0.205 & 0.103 & 0.003 \\
\midrule
7 & \textbf{0.028} & 0.362 & 0.723 & 0.324 & 0.438 & 0.827 & 0.227 & 0.097 & 0.003 \\
\bottomrule
\end{tabular}
}
\caption{A comparison of LS and NExT's two versions, trained on monotonic data, NExT Monotonic, and trained on the whole function space, NExT Whole Space. NExT Monotonic version outperforms both NExT Whole Space and LS, showing how it manages to learn complex underlining structures that are not present in the basis functions.}
\label{table_monotonic}
\end{table*}

Table \ref{table_monotonic} shows the results of the monotonic noisy function space experiment. LS results are 0.124, 0.205, and 0.227 on the noisy monotonic Chebyshev polynomials of degrees 3,5 and 7, respectively. NExT without a focus on monotonic function, NExT Whole Space, scored 0.012, 0.062, and 0.324, respectively. NExT, with a focus on monotonic functions, NExT Monotonic, scored 0.016, 0.022, and 0.028, respectively. On the 3-degree noisy monotonic polynomials, both NExT Whole Space and NExT Monotonic outperformed LS but achieved similar results where NExT Whole Space performed better. We attribute this performance gain to the size of the function space, as 3-degree polynomials are small enough that NExT, even without a more specific subspace, performs well. On the 5 and 7 noisy monotonic polynomials, NExT Monotonic outperforms both LS and NExT Whole Space by a considerable margin. NExT Monotonic achieved an error reduction rate of 64.5\% and 91.4\% over NExT Whole Space and an error reduction rate of 89.3\% and 87.7\% over LS. These results attest to the importance of letting NExT learn about the specific function space with a complex underlying structure and not the vector space containing it to enhance its extrapolation capabilities further. Although Fig. \ref{fig_LS_vs_ie_7_monotonic} shows how NExT Monotonic does not necessarily predict monotonic functions in the whole space, it did predict a monotonic function in the interest area $\Xi$.

\subsection{Extrapolating with anchor functions experiments} \label{sec:anchor_experiment}

We evaluate NExT's ability to solve the anchored extrapolation problem of Definition~\ref{def_anchor_function_problem}. The function to extrapolate is:

\begin{equation}
    f(x) = 0.8^x-\cos(x)+2\sin(2x)+\frac{1}{x+1} .
    \label{eq_extrapolation_function_1}
\end{equation}
For this part, we use $\Omega=[0,1.5 \pi)$ and $\Xi=[1.5\pi,2\pi]$. We also assume anchor functions are given, and specifically, we utilize two sets of three anchor functions for evaluation purposes. The two sets have the property of including either decaying or non-decaying functions. Decaying anchor functions have a decaying addition to~\eqref{eq_extrapolation_function_1}, whereas non-decaying anchor functions have an addition of growing functions. Both sets of anchor functions are given in Table. \ref{tbl:anchor_functions} and depicted in Fig. \ref{fig_decaying_functions}, \ref{fig_non_decaying_functions}.

\begin{table*}[htbp]
    \centering
    \resizebox{0.5\textwidth}{!}{
        \begin{tabular}{@{}clc@{}}
            \toprule
            \textbf{Anchor function type} & \textbf{Anchor function} & $\boldsymbol{\Xi}$ \textbf{RMSE} \\
        \hline
        Decaying & $f(x)+\frac{2}{x+1}$ & $0.310$ \\
        & $f(x)+\frac{3 \sin(x)}{x+1}$ & $0.345$ \\
        & $f(x)+0.9^x$ & $0.562$ \\
        \hline
        Non-decaying & $f(x)+\frac{x}{10}$ & $0.552$ \\
        & $f(x)+\sin^2(x)$ & $0.613$ \\
        & $f(x)+\frac{\log^2(x+1)}{5}$ & $0.702$ \\
        \hline
        \end{tabular}
    }
\caption{Different anchor functions and their RMSE score in the extrapolation area.}
\label{tbl:anchor_functions}
\end{table*}

\begin{figure}[ht]
    \centering
    \begin{subfigure}[]{0.32\textwidth}
        \includegraphics[width=\textwidth]{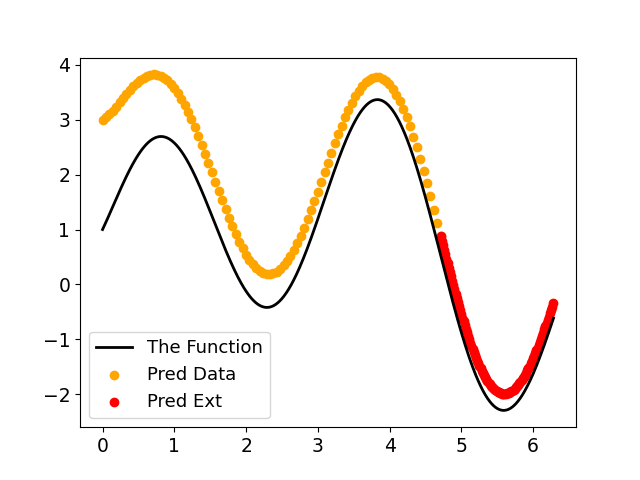}
        \caption*{}    
    \end{subfigure}
    \begin{subfigure}[]{0.32\textwidth}
        \includegraphics[width=\textwidth]{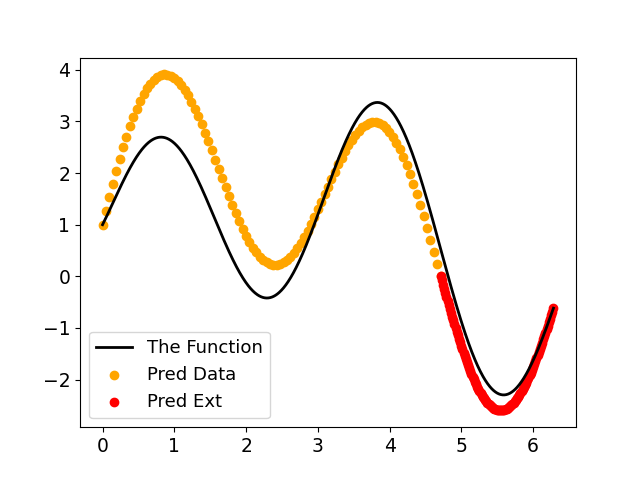}
        \caption*{}  
    \end{subfigure}
    \begin{subfigure}[]{0.32\textwidth}
        \includegraphics[width=\textwidth]{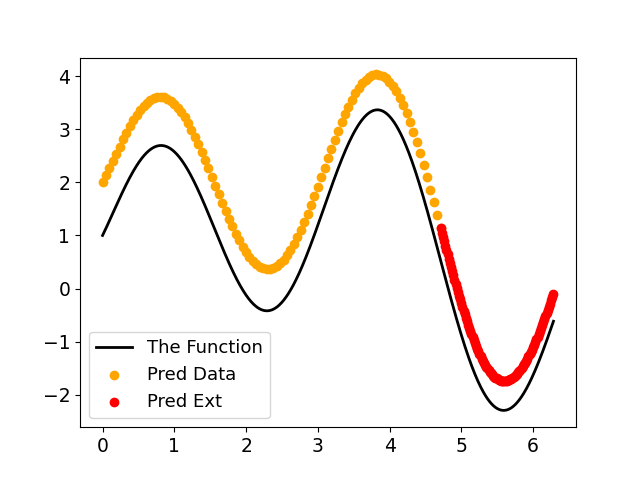}
        \caption*{}  
    \end{subfigure}
\caption{Anchor functions with decaying functions. Plotted next to the wished extrapolation function $f$. $f(x)+\frac{2}{x+1}$ (left), $f(x)+\frac{3 \sin(x)}{x+1}$ (middle), $f(x)+0.9^x$ (right), with RMSE scores of 0.310, 0.345, and 0.562 respectively.}
\label{fig_decaying_functions}
\end{figure}

\begin{figure}[ht]
    \centering
    \begin{subfigure}[]{0.32\textwidth}
        \includegraphics[width=\textwidth]{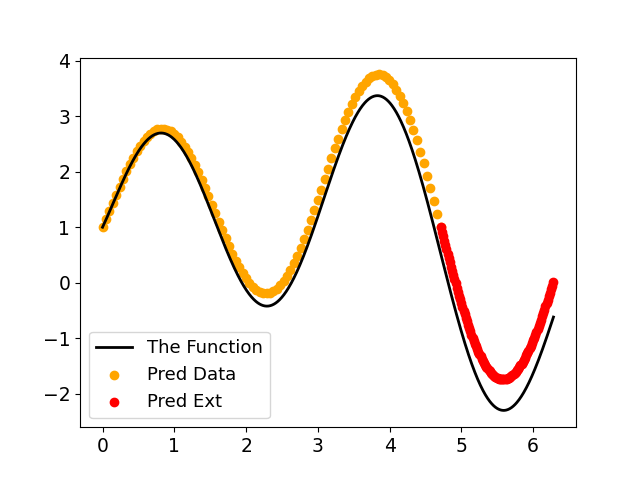}
        \caption*{}    
    \end{subfigure}
    \begin{subfigure}[]{0.32\textwidth}
        \includegraphics[width=\textwidth]{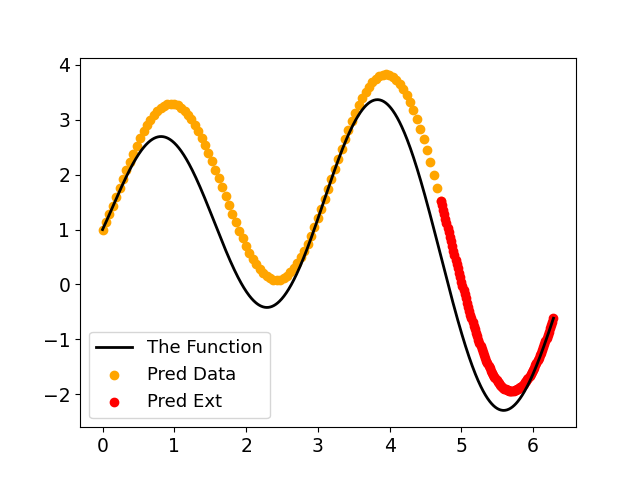}
        \caption*{}  
    \end{subfigure}
    \begin{subfigure}[]{0.32\textwidth}
        \includegraphics[width=\textwidth]{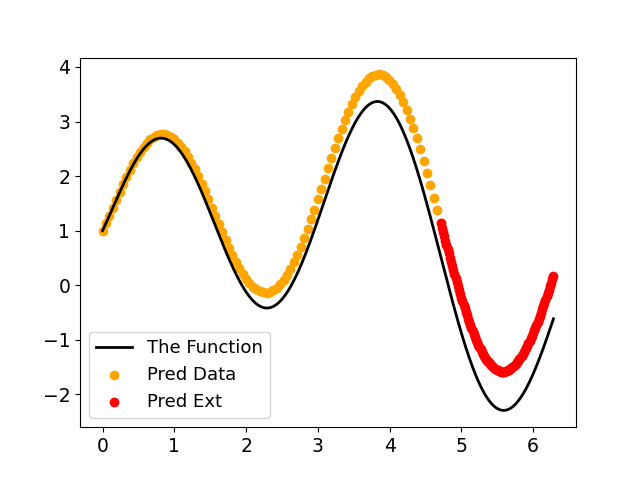}
        \caption*{}  
    \end{subfigure}
\caption{Anchor functions with non-decaying functions. Plotted next to the wished extrapolation function $f$. $f(x)+\frac{x}{10}$ (left), $f(x)+\sin^2(x)$ (middle),  and $f(x)+\frac{\log^2(x+1)}{5}$ (right), with RMSE scores of 0.552, 0.613, and 0.702 respectively.}
\label{fig_non_decaying_functions}
\end{figure}

Since each anchor function is known to be close to $f$, it is possible only to use one of them for extrapolation, but since the underlying extrapolation is unknown, it is not possible to choose which is preferred. Therefore, a successful algorithm should, in theory, outperform the mean error rate, and an algorithm that has an RMSE score that is lower than the lowest RMSE score of all anchor functions will be considered an achievement. As mentioned in Section~\ref{sec:anchor_extrapolation}, NExT proposes to use the anchor functions given as a frame to create the function space $\mathcal{F}$ it will learn to extrapolate.

\begin{figure}[ht]
    \centering
        \begin{subfigure}[b]{0.475\textwidth}
            \centering
            \includegraphics[width=\textwidth]{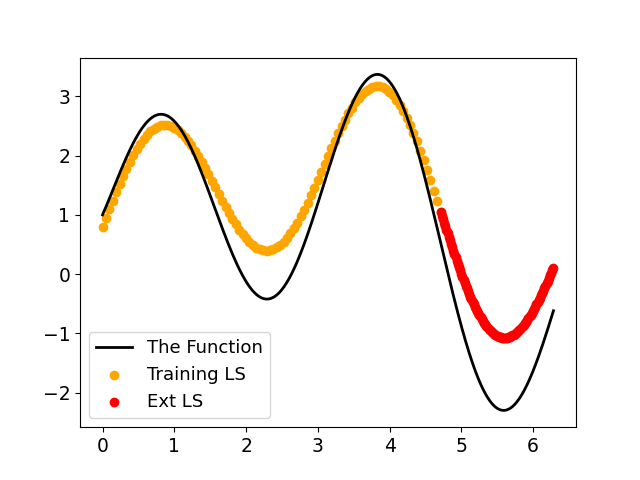}
        \end{subfigure}
        \hfill
        \begin{subfigure}[b]{0.475\textwidth}  
            \centering 
            \includegraphics[width=\textwidth]{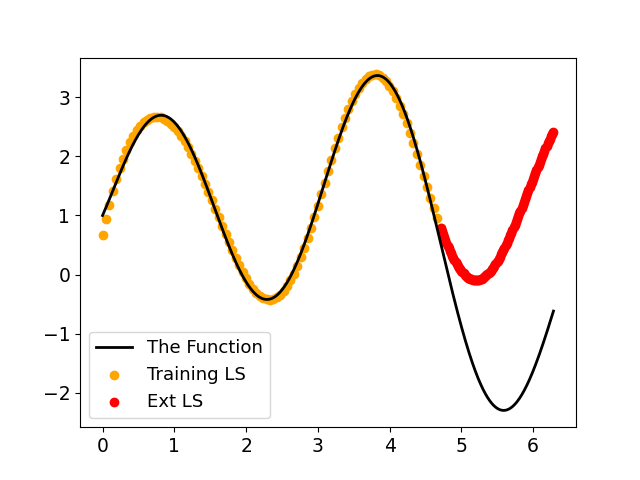} 
        \end{subfigure}
        \vskip\baselineskip
        \begin{subfigure}[b]{0.475\textwidth}   
            \centering 
            \includegraphics[width=\textwidth]{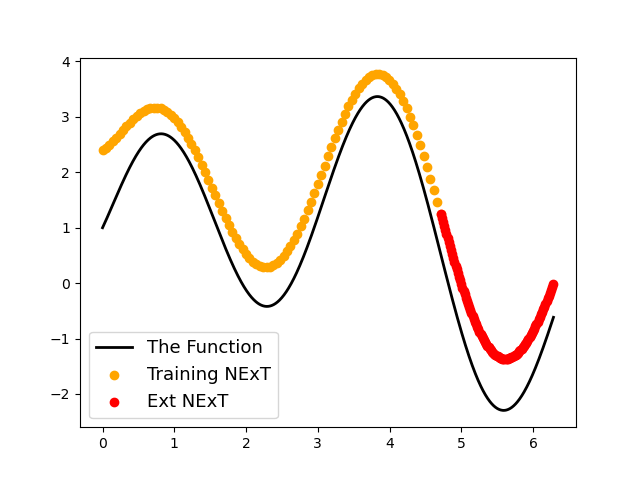}  
        \end{subfigure}
        \hfill
        \begin{subfigure}[b]{0.475\textwidth}   
            \centering 
            \includegraphics[width=\textwidth]{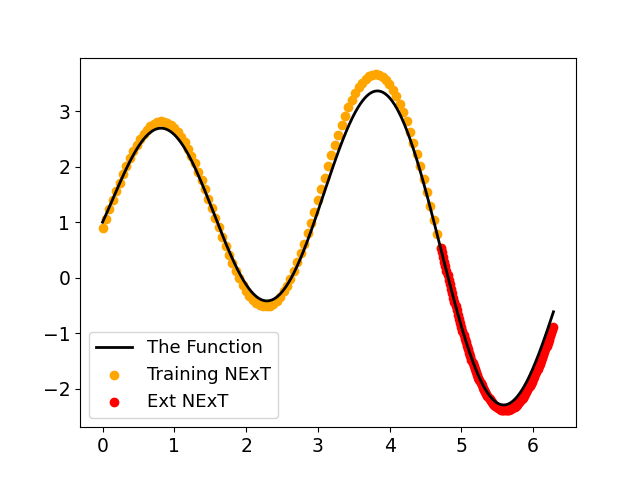}
        \end{subfigure}
\caption{Results for extrapolating \eqref{eq_extrapolation_function_1} for LS (top) and NExT (bottom) using a frame. Left figures use the 3 decaying functions only, and right figures use the 3 decaying functions with an additional 7 basis elements from the trigonometric functions. LS RMSE scores are 1.02 and 2.399, respectively, for without and with filler functions, and NExT scores are 0.844 and 0.131, respectively.}
\label{fig_unknown_function_with_decaying_functions}
\end{figure}

\begin{figure}[ht]
    \centering
        \begin{subfigure}[b]{0.475\textwidth}
            \centering
            \includegraphics[width=\textwidth]{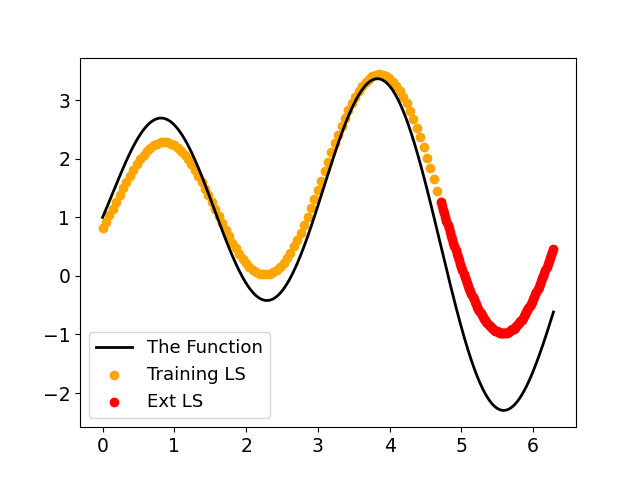}
        \end{subfigure}
        \hfill
        \begin{subfigure}[b]{0.475\textwidth}  
            \centering 
            \includegraphics[width=\textwidth]{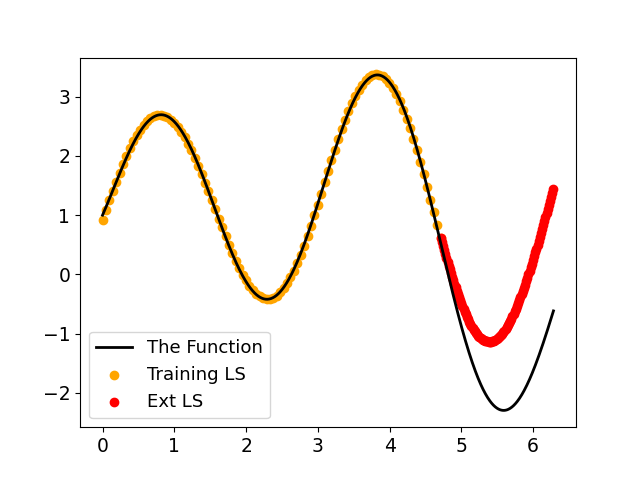} 
        \end{subfigure}
        \vskip\baselineskip
        \begin{subfigure}[b]{0.475\textwidth}   
            \centering 
            \includegraphics[width=\textwidth]{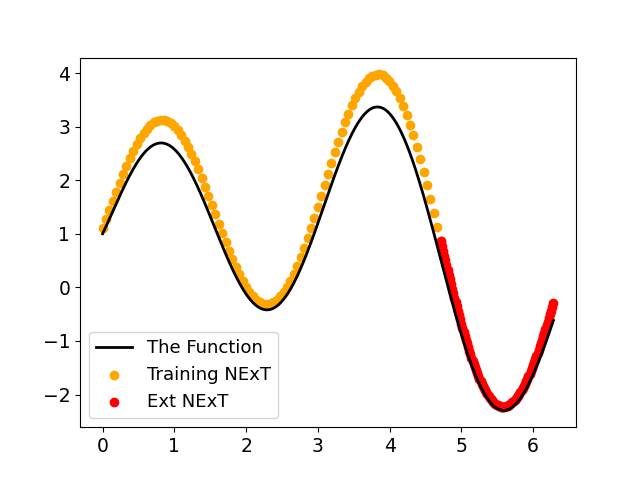}  
        \end{subfigure}
        \hfill
        \begin{subfigure}[b]{0.475\textwidth}   
            \centering 
            \includegraphics[width=\textwidth]{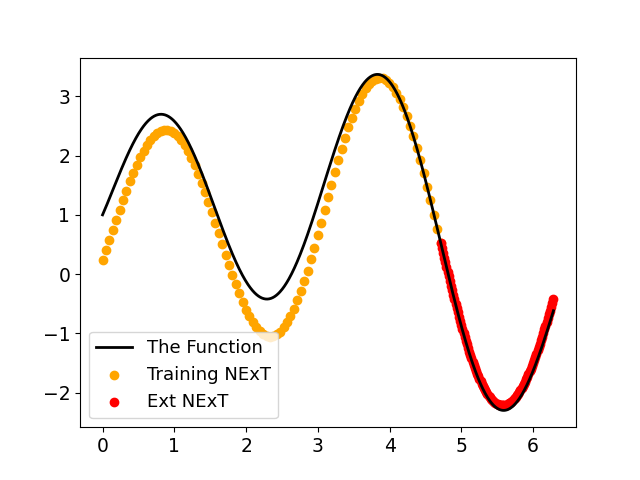}
        \end{subfigure}
\caption{Results for extrapolating \eqref{eq_extrapolation_function_1} for LS (top) and NExT (bottom) using a frame. Left figures use the 3 decaying functions only, and right figures use the 3 decaying functions with an additional 7 basis elements from the trigonometric functions. LS RMSE scores are 1.172 and 1.309, respectively, for without and with filler functions, and NExT scores are 0.183 and 0.098, respectively.}
\label{fig_unknown_function_with_non_decaying_functions}
\end{figure}

Figures \ref{fig_unknown_function_with_decaying_functions} and \ref{fig_unknown_function_with_non_decaying_functions} present the results of the LS model and NExT for both decaying and non-decaying anchor function frames. In these experiments, we utilized seven filler functions from the trigonometric basis. The condition numbers ($\kappa$) for the anchor function frames were 0.6628 and 0.38174 for the decaying and non-decaying functions, respectively. With the inclusion of filler functions, these values increased to 24.2598 and 12.4486, respectively.

For the decaying anchor function problem, LS yielded extrapolation results of 1.02 without filler functions and 2.399 with filler functions, while NExT achieved scores of 0.844 and 0.131, respectively. This translates to error reduction rates of 17.3\% and 94.5\%, respectively, and an overall reduction rate of 87.2\% compared to the best LS result using three anchor frames. Notably, NExT outperformed all individual anchor functions when filler functions were included, achieving a reduction of 57.7\% from the best error rate.

Regarding non-decaying functions, NExT exhibited even better performance, with extrapolation results of 0.183 and 0.098 compared to LS's 1.172 and 1.309. This represents reduction rates of 84.4\% and 92.5\%, respectively, and a reduction of 91.6\% from the best LS error rate. Additionally, NExT achieved an error reduction rate of 82.2\% compared to the closest anchor function.

Based on the condition numbers $\kappa$, it seems that including filler functions in the frame makes the extrapolation more challenging, which seems to go against their intended purpose. This observation holds true for LS, as seen in its poorer performance compared to LS without filler functions. However, NExT, by focusing on minimizing $\Xi$ instead of $\Omega$, renders the condition number irrelevant. Unlike LS, NExT performs well with filler functions, highlighting its ability to prioritize minimizing the extrapolation area $\Xi$ rather than the approximation area $\Omega$. This underscores the importance of defining the right objective function. Consequently, we can conclude that NExT performs well, surpassing LS in both scenarios and outperforming each individual anchor function, thus demonstrating its usefulness in addressing \ref{def_anchor_function_problem}. We note that in this scenario, the conditions of Theorem~\ref{theorem_1} do not hold, even more so than in Section~\ref{sec:global_function_space}. Therefore, LS does not outperform even though the problem has condition numbers lower than 1.

\begin{remark}[Using $\kappa$]
    We recall the LS results over the anchor function problem Definition~\ref{def_anchor_function_problem} and inspect the values of $\kappa$ of~\eqref{eq:least_squares_extrapolation_condition_number}, even on problems that do not fully fit the conditions such as in Theorem~\ref{theorem_1}. On one hand, the resulting $\kappa$ for each problem were $0.6628, 0.38174, 24.2598, and 12.4486$ for the decaying, non-decaying, decaying with filler functions, and non-decaying with filler functions, respectively. On the other hand, the LS obtained RMSE scores of $1.02, 1.172, 2.399$, and $1.309$ for the decaying, non-decaying, decaying with filler functions, and non-decaying with filler functions, respectively. Therefore, there is a clear correlation between $\kappa$ and LS's RMSE score, attesting to the ability to use $\kappa$ to choose the basis for extrapolation, evening a more general settings than we used for the formal theorems.
    \label{remark:condition_number_anchor_functions}
\end{remark}

\begin{remark}[Anchor functions]
    As shown in the results in Fig. \ref{fig_unknown_function_with_decaying_functions}, \ref{fig_unknown_function_with_non_decaying_functions}, NExT manages to improve the RMSE of the anchor functions used. In scenarios where anchor functions are harder to come by, it is possible to use different extrapolation techniques to create anchor functions. This is out of the scope of this paper and is left for further research.
\end{remark}

\subsection{Deep learning experimental results}
\label{sec:deep_learning_results}
To fully evaluate NExT, we continue to compare it with two modern deep learning models. Therefore, we compare the results of Snake and ReLU activation functions on the Global Function Space and to the anchor function problem Definition~\ref{def_anchor_function_problem}. We test both models on the Chebyshev noisy polynomials as in Section \ref{sec:global_function_space}.

\begin{table*}[htbp]
\centering
\resizebox{1\textwidth}{!}{
\begin{tabular}{@{}cccccccccccc@{}}
\toprule
&   
\multicolumn{2}{c}{NExT} & 
\multicolumn{2}{c}{ReLU Net}  &
\multicolumn{2}{c}{Snake} \\
\cmidrule(lr){2-3}\cmidrule(lr){4-5}\cmidrule(lr){6-7}
Num coefficients & $\Xi$ RMSE & $\Omega$ RMSE & $\Xi$ RMSE & $\Omega$ RMSE & $\Xi$ RMSE & $\Omega$ RMSE \\
\midrule
3 & \textbf{0.021} &  0.598 & 0.361 &  0.005 & 0.095 & 0.004 \\
\midrule
5 & \textbf{0.054} & 0.597 & 0.705 &  0.030 & 0.628 & 0.015 \\
\midrule
7 & \textbf{0.200} & 0.605 & 0.874 & 0.055 & 0.815 & 0.056 \\
\bottomrule
\end{tabular}
}
\caption{A comparison between NExT, ReLU nets and Snake on 3,5, and 7 degree Chebyshev polynomials.}
\label{tbl:experiments_deep_learning_regular}
\end{table*}

The results on the noisy Chebyshev polynomials are given in Table~\ref{tbl:experiments_deep_learning_regular}. Snake RMSE scores are 0.095, 0.628, and 0.815 on the 3,5, and 7-degree polynomials, respectively. ReLU RMSE scores were worse, with 0.361, 0.705, and 0.874, respectively. NExT outperforms both deep learning models with scores of 0.021, 0.054, and 0.200. Resulting in error reduction rates of 77.9\%, 91.4\%, and 75.5\% on the 3,5, and 7-degree noisy Chebyshev polynomials.

\begin{figure}[ht]
    \centering
    \begin{subfigure}[]{0.32\textwidth}
        \includegraphics[width=\textwidth]{figures/next_anchor_function_10_non_decaying.png}
        \caption*{}    
    \end{subfigure}
    \begin{subfigure}[]{0.32\textwidth}
        \includegraphics[width=\textwidth]{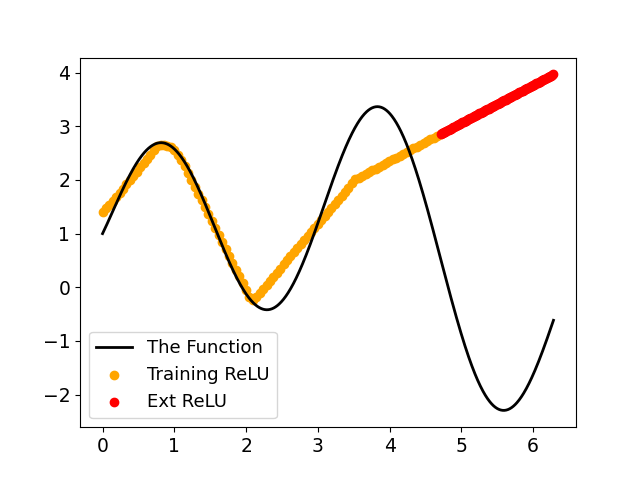}
        \caption*{}  
    \end{subfigure}
    \begin{subfigure}[]{0.32\textwidth}
        \includegraphics[width=\textwidth]{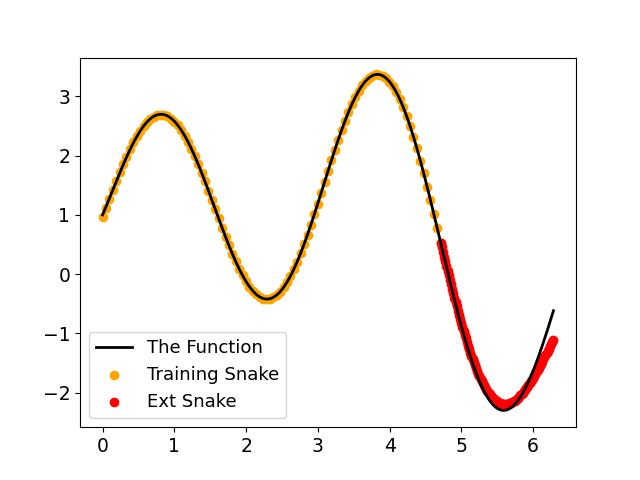}
        \caption*{}  
    \end{subfigure}
\caption{A comparison of extrapolating \eqref{eq_extrapolation_function_1} for NExT while using the non-decaying anchor function frame with filler functions (left), ReLU network (middle), and Snake (right). NExT RMSE score is 0.098, while Snake and ReLU's scores are 4.954 and 0.150 respectively. Snake reported results while learning the frequency and is better while not, which achieved a score of 0.363.}
\label{fig_unknown_function_with_deep_learning_models}
\end{figure}

The results of extrapolating \eqref{eq_extrapolation_function_1}, are in Fig. \ref{fig_unknown_function_with_deep_learning_models}. Snake relatively performs well with an RMSE score of 0.150 compared to the 4.954 RMSE score of the ReLU net. Nonetheless, NExT still manages to outperform Snake-net and obtain 0.131 and 0.098 scores with the decaying and non-decaying anchor functions, respectively, resulting in 12.7\% and 34.7\% error reduction rates, respectively. For Snake, since there is no simple method in choosing whether to learn or not to learn the frequency, the reported results of Snake consist of the best value of the two choices, which are to learn the frequency; otherwise, Snake with a constant frequency of 1 achieved an RMSE score of 0.363.

\subsection{Extrapolation over ``far'' domains}

While devising our methodology, we encountered no constraints limiting us to a narrow vicinity. As NExT, adeptly learns to extrapolate to a particular $\Xi$, we conducted tests encompassing more distant areas. The problem is extrapolating \eqref{eq_extrapolation_function_1} as in Section \ref{sec:anchor_experiment} with areas of a distance of 1,3, and 7 from $\Omega$, with the nondecaying anchor functions. NExT will use the frame with the 7 filler functions, and LS will use the frame that works best between the frame with the anchor functions alone or with the 7 filler functions. The distances represent very far areas compared to the learning area, which is 1.5$\pi$.

\begin{figure}[!ht]
    \centering
    \begin{subfigure}[]{\textwidth}
        \includegraphics[width=.32\textwidth]{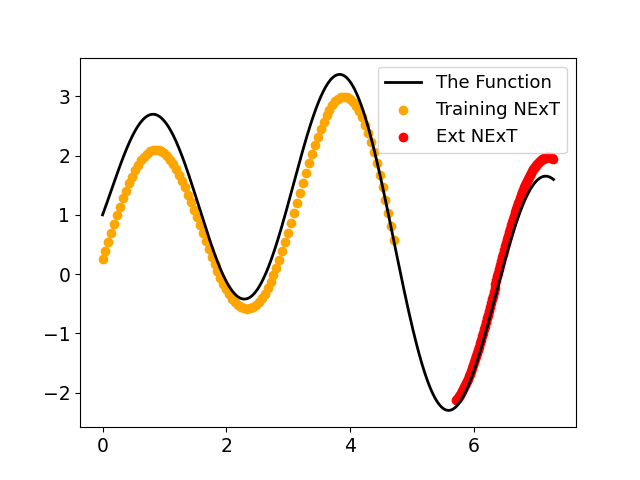}
        \includegraphics[width=.32\textwidth]{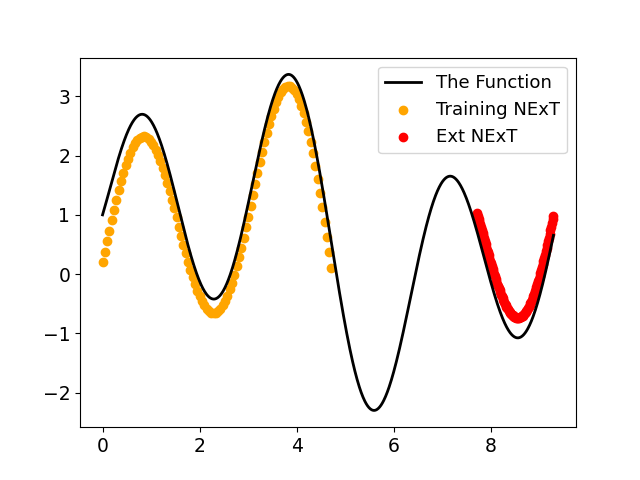}
        \includegraphics[width=.32\textwidth]{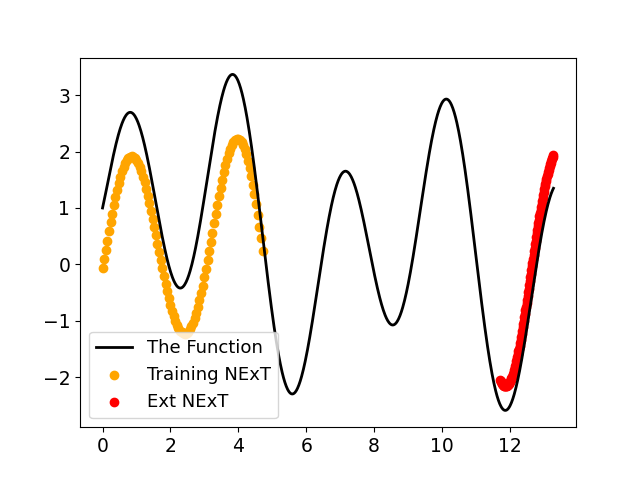}
        \caption*{Our method}  
    \end{subfigure}
    \begin{subfigure}[]{\textwidth}
        \includegraphics[width=.32\textwidth]{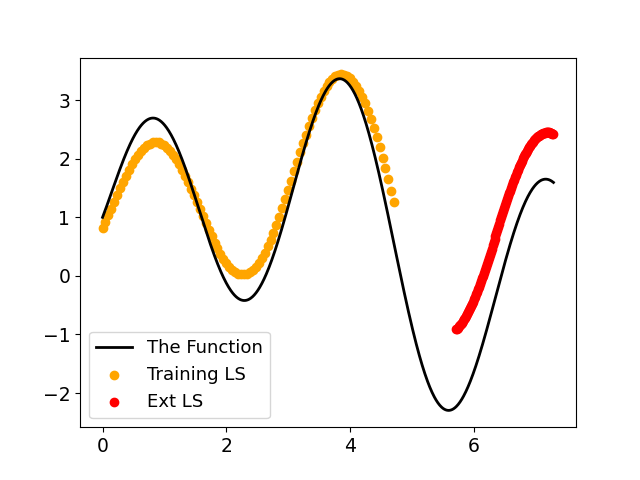}
        \includegraphics[width=.32\textwidth]{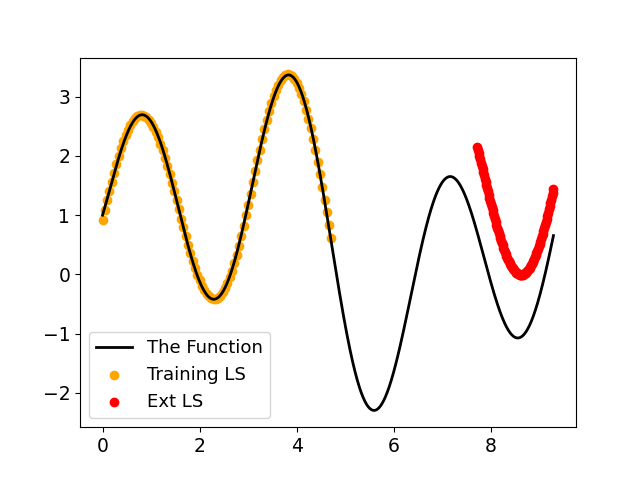}
        \includegraphics[width=.32\textwidth]{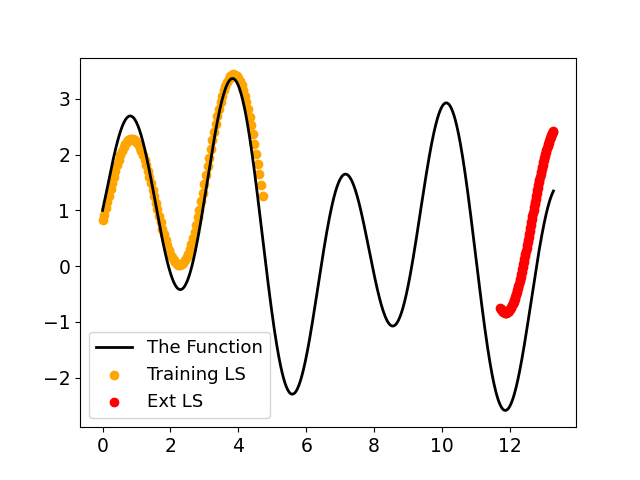}
        \caption*{Least squares}  
    \end{subfigure}
    \begin{subfigure}[]{\textwidth}
        \includegraphics[width=.32\textwidth]{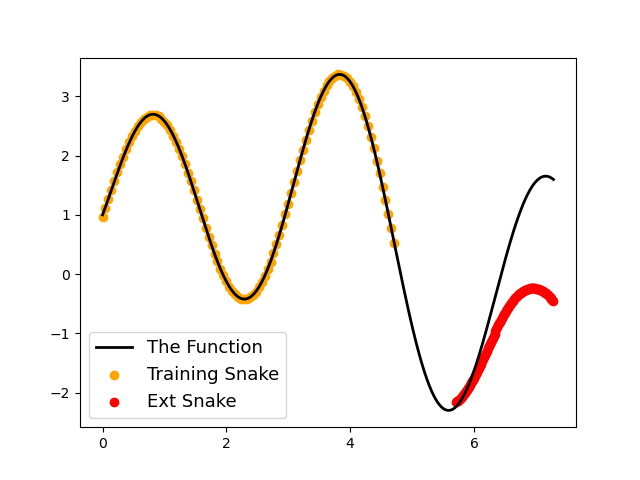}
        \includegraphics[width=.32\textwidth]{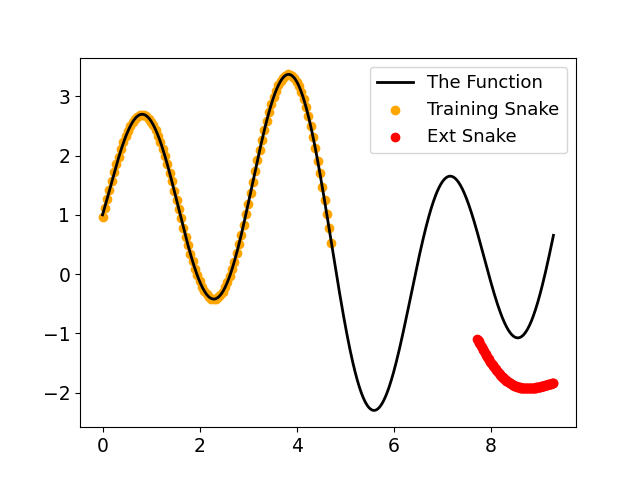}
        \includegraphics[width=.32\textwidth]{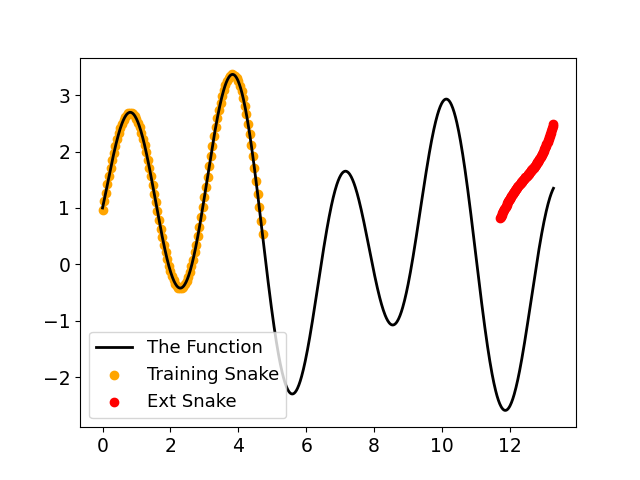}
        \caption*{Snake}  
    \end{subfigure}
\caption{A comparison of extrapolating \eqref{eq_extrapolation_function_1} for NExT while using the non-decaying anchor function frame with filler functions (top), LS with the best frame (middle), and Snake net (bottom). Left, middle, and right rows indicate 1,3 and 7 distance of $\Omega$ to $\Xi$. The results for distance 1.0 are 0.225, 1.020, and 1.183 for NExT, LS, and Snake, respectively. For distances 3.0, 0.333, 1.126, and 1.387, respectively. For distances 7.0, 0.492, 1.478, 2.705, respectively.}
\label{fig:fatherawayanchorfunctions}
\end{figure}

The results of the farther away areas are in Fig. \ref{fig:fatherawayanchorfunctions}. NExT performs considerably better than LS and Snake, as the extrapolation area $\Xi$ information is incorporated into its learning algorithm. The results are 0.225, 0.333, and 0.492 for NExT with 1,3 and 7 distances, respectively. For LS is 1.020, 1.126, and 1.478, respectively, and Snake achieved scores of 1.183, 1.387, and 2.705, respectively. As the $\Xi$ is farther away from $\Omega$ all algorithms result deteriorate but NExT score suffers the least deterioration. NExT achieved an error reduction rate of 77.9\%, 70.4\%, and 66.7\% from the best scores LS or Snake achieved for the 1,3 and 7 distances, respectively. Snake performance deteriorated the most from the close area in Section \ref{sec:deep_learning_results}. Inline with results other neural networks encounter during extrapolation \cite{ziyin2020neural, belcak2022periodic, parascandolo2016taming}, attesting to the importance of using a well-defined frame that its extrapolation is well understood. In addition, as Snake does not use the anchor functions and lacks the same assumptions LS and NExT have on the extrapolated function, this result was anticipated. The anchor function RMSE for $f(x)+\frac{x}{10}$ is 0.651, 0.851, and 1.251, for the 1,3, and 7 distance between $\Omega$ and $\Xi$. For $f(x) + sin^2(x)$, 0.296, 0.681, and 0.248. And lastly for $f(x)+\frac{log^2(x+1)}{5}$, 0.812, 1.014, and 1.355. Thus, NExT improved all individual anchor functions in the 1 and 3 area distances and the mean RMSE in the 7 distance. Attesting to NExT's strength in solving farther away areas.

\subsection{Noise sensitivity} \label{sec:noise_sensitivity}
To showcase NExT's ability to withstand different noisy environments over baseline models, sensitivity to different noise levels were tested. We used different noise levels, SNR=20,35, and 50, and recorded each algorithm's degradation on 100 Chebyshev 5-degree polynomials generated as in Section \ref{sec:global_function_space}.

\begin{table*}[htbp]
\centering
\resizebox{1\textwidth}{!}{
\begin{tabular}{@{}ccccccccccc@{}}
\toprule
&   
\multicolumn{3}{c}{NExT} & 
\multicolumn{3}{c}{LS}\\
\cmidrule(lr){2-4}\cmidrule(lr){5-7}
SNR & $\Xi$ RMSE & Coefficients RMSE & $\Omega$ RMSE & $\Xi$ RMSE & Coefficients RMSE & $\Omega$ RMSE \\
\midrule
No Noise & 0.023 & 0.351 & 0.628 & $\mathbf{8.05 \times 10^{-7}}$ & $3.41\times 10^{-7}$ & $6.13\times 10^{-7}$\\
\midrule
50 & \textbf{0.023} & 0.319 & 0.549 & 0.042 & 0.018 & 0.001\\
\midrule
40 & \textbf{0.033} & 0.333 & 0.596 & 0.164 & 0.070 & 0.002 \\
\midrule
35 & \textbf{0.054} & 0.334 & 0.597 & 0.291 & 0.124 & 0.004 \\
\midrule
30 & \textbf{0.105} & 0.327 & 0.570 & 0.438 & 0.187 & 0.006 \\
\midrule
20 & \textbf{0.402} & 0.415 & 0.607 & 1.357 & 0.577 & 0.017 \\
\bottomrule
\end{tabular}
}
\caption{A comparison between NExT and LS noisy sensitivity. NExT was trained on SNR=35 only.}
\label{tbl:experiments_noise_sensitivity}
\end{table*}

The noise sensitivity examples appear in Table~\ref{tbl:experiments_noise_sensitivity}. On functions without noise, LS finds a near-perfect fit, but with as little as SNR=50, its result drops from $8.05\times 10^{-7}$ to 0.042. This trend continues as we increase the noise level. NExT, on the other hand, with the presence of no noises, achieves 0.023, but this error rate does not deteriorate rapidly, resulting in outperforming LS on every noise level we checked. NExT's poor performance on the no-noise data set can be attributed to it being trained on SNR=35 noise only and to the lack of precise convergence with gradient descent with batches~\cite{cacciola2023convergence}.

\subsection{Function extrapolation over a manifold}

This subsection illustrates the generality of our methodology for addressing extrapolation over manifold domains. In this case, as a compact manifold, we use the sphere. This demonstration also shows, by definition, multivariate domains with a larger number of input parameters. In this case, we conducted an experiment using real-valued spherical harmonics as basis functions~\cite{schonefeld2005spherical}. The standard definition of spherical harmonics is
\begin{equation}
    Y^m_l(\theta,\phi)=\sqrt{(2l+1)\frac{(l-m)!}{(l+m)!}}P^m_l(cos\theta)e^{im\phi} ,
    \label{eq:spherical_harmonics}
\end{equation}
where $l=0,1,2...$, $-l\leq m \leq l$, $P^m_l$ the associated Legendre polynomial \cite{weisstein2011associated}, and $\theta$ and $\phi$ are the spherical coordinates. $l$ is considered the degree of the polynomial, meaning that, for example, polynomials of degree 2 contain basis elements with at most $l=2$ and can contain $9$ basis elements ($2l+1$ for each degree $l$). Spherical harmonics in the form \eqref{eq:spherical_harmonics} are complex-valued functions. Therefore, we use a simple manipulation to derive the real-valued spherical harmonics:
\begin{equation}
Y_{lm}(\theta, \phi) = 
\begin{cases}
\frac{1}{\sqrt{2}}( Y^m_l(\theta,\phi)+(-1)^m Y^{-m}_l(\theta,\phi)), & m > 0\\
 Y^m_l(\theta,\phi), & m = 0 \\
\frac{1}{i\sqrt{2}}( Y^{-m}_l(\theta,\phi)-(-1)^m Y^{m}_l(\theta,\phi)), & m < 0
\end{cases}
\label{eq:spherical_harmonics_real}
\end{equation}

The polynomials were arranged in the following order: $(l=0, m=0), (l=1, m=1), (l=1, m=0), (l=1, m=-1), (l=2, m=2),$ and so on, establishing an order for the polynomial basis. The extrapolation extends to the upper hemisphere, with the training confined to the bottom third of the sphere. An illustration of the corresponding regions is given in Fig \ref{fig:spherical_harmonics_extraplation_and_trianing}. The resulting $\kappa$ value is 41.027, suggesting that the LS may perform inadequately, while NExT is anticipated to overcome such challenge. The training set comprised of a 100-point grid, of equally spaced points in each dimension, while the extrapolation area consisted of a 10,000-point equally spaced grid.

\begin{figure}[htp]
    \centering
        \includegraphics[width=0.6\linewidth]{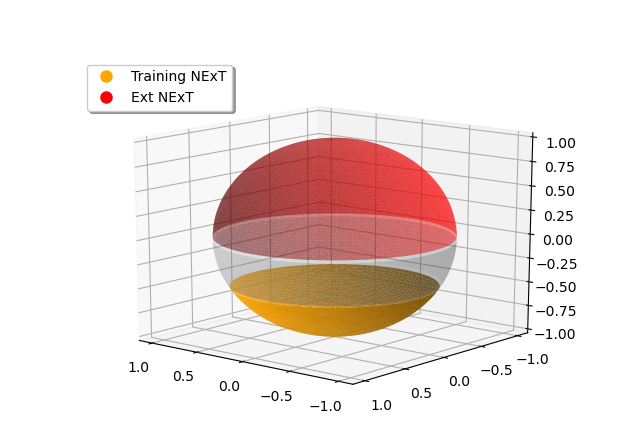}
\caption{A visual representation of the spherical harmonics extrapolation problem depicts distinct regions for training and extrapolation. The training area is confined to the lower third of the sphere, highlighted in orange, while the extrapolation area encompasses the upper hemisphere in red. The intermediate grayed area delineates the space between the training and extrapolation regions.}
\label{fig:spherical_harmonics_extraplation_and_trianing}
\end{figure}

\begin{table*}[htbp]
\centering
\resizebox{1\textwidth}{!}{
\begin{tabular}{@{}ccccccccccc@{}}
\toprule
&   
\multicolumn{3}{c}{NExT} & 
\multicolumn{3}{c}{LS}\\
\cmidrule(lr){2-4}\cmidrule(lr){5-7}
Num coefficients & $\Xi$ RMSE & Coefficients RMSE & $\Omega$ RMSE & $\Xi$ RMSE & Coefficients RMSE & $\Omega$ RMSE \\
\midrule
5 & \textbf{0.013} & 0.017 & 0.006 & 0.911 & 0.635 & 0.009\\
\midrule
9 & \textbf{0.046} & 0.035 & 0.009 & 0.875 & 0.609 & 0.009 \\
\bottomrule
\end{tabular}
}
\caption{A comparison between NExT and LS. The RMSE over $\Xi$ indicates that NExT outperforms the LS method. Unlike the one-dimensional case, NExT manages to present comparable results to the LS on $\Omega$.}
\label{tbl:experiments_spherical_harmonics}
\end{table*}

\begin{figure}[htp]
    \centering
    \begin{subfigure}[t]{0.45\textwidth}
        \includegraphics[width=\textwidth]{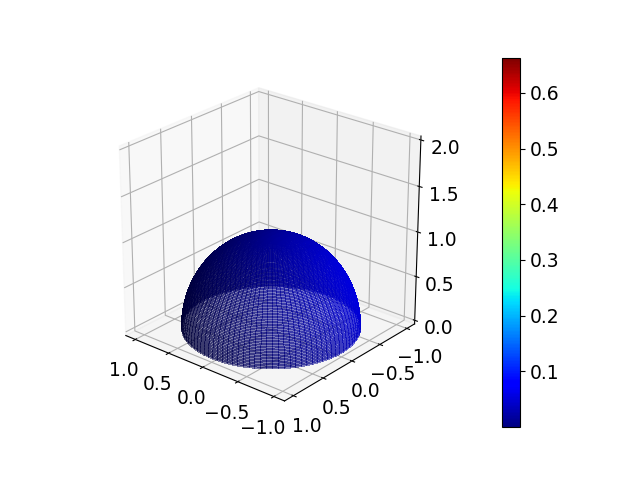}
        \caption*{}    
    \end{subfigure}\qquad
    \begin{subfigure}[t]{0.45\textwidth}
        \includegraphics[width=\textwidth]{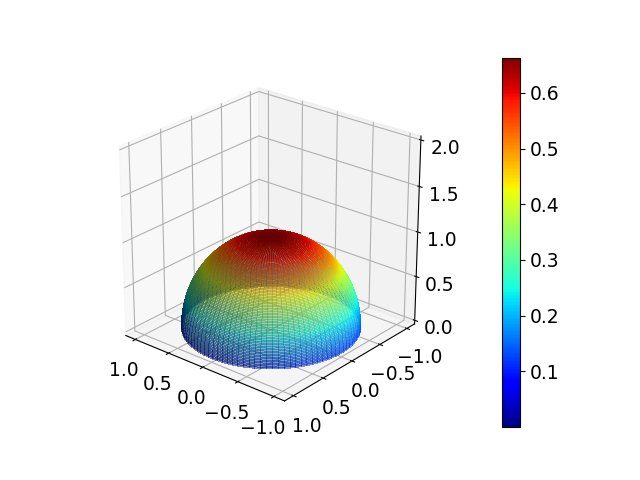}
        \caption*{}  
    \end{subfigure}
\caption{An error plot comparing our model (left) and an LS-based extrapolation (right) for a 2-degree (9 basis elements) real-valued spherical harmonic polynomial. Brighter colors stand for higher errors. LS achieves an RMSE score of 0.471, while NExT outperforms it with 0.036, achieving a 92.3\% error reduction rate.}
\label{fig:spherical_harmonics}
\end{figure}

\begin{figure}[htp]
    \centering
    \begin{subfigure}[t]{0.32\textwidth}
        \includegraphics[width=\textwidth]{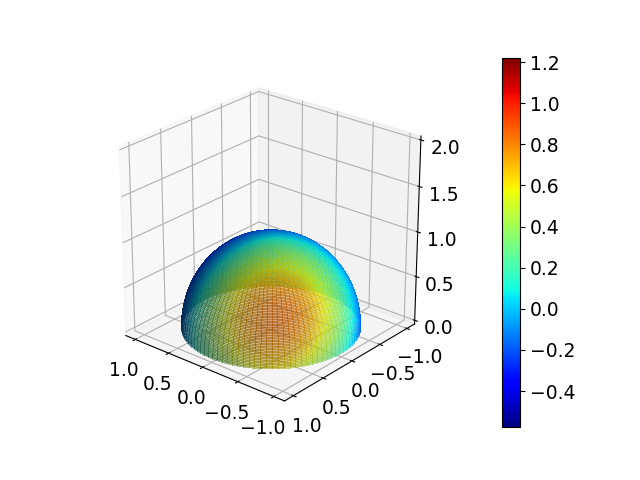}
        \caption*{}    
    \end{subfigure}
    \begin{subfigure}[t]{0.32\textwidth}
        \includegraphics[width=\textwidth]{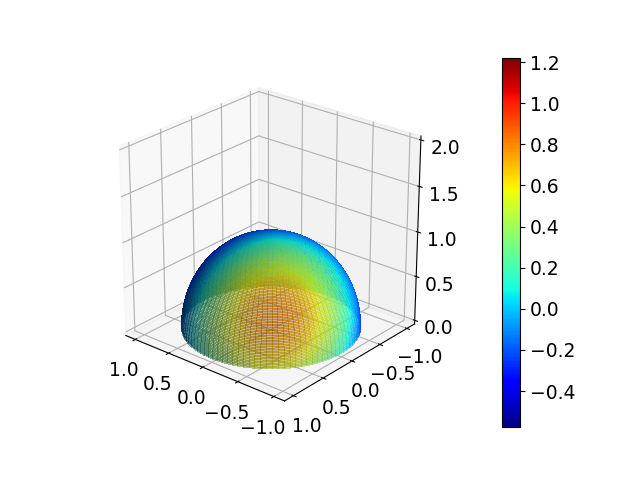}
        \caption*{}  
    \end{subfigure}
    \begin{subfigure}[t]{0.32\textwidth}
        \includegraphics[width=\textwidth]{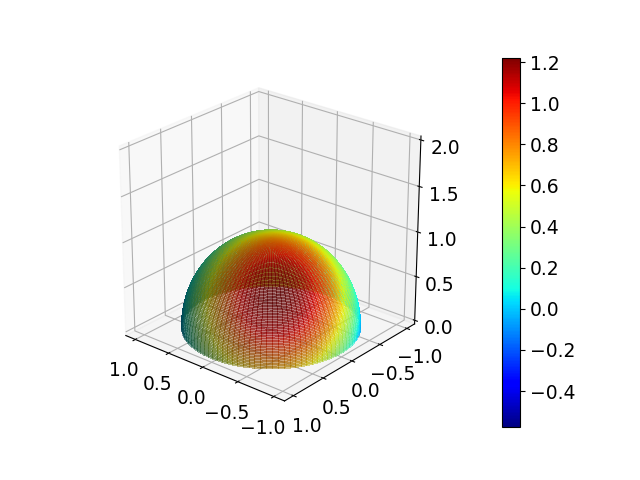}
        \caption*{}  
    \end{subfigure}
\caption{Plot comparing true function (left) to our model (middle) and a LS extrapolation (right) for a 2-degree (9 basis elements) real-valued spherical harmonic polynomial. LS achieves an RMSE score of 0.471, while NExT outperforms it with 0.036, achieving a 92.3\% error reduction rate.}
\label{fig:spherical_harmonics_real_function}
\end{figure}

Both NExT and LS predicted 2-degree polynomials (9 basis elements) and were given data sets containing 100 1-degree (5 basis elements) and 2-degree polynomials randomly sampled with $r_\sigma=0.25$ and $r_m=1$. The results are in Table. \ref{tbl:experiments_spherical_harmonics}. NExT outperforms LS in all polynomial degrees, achieving error reduction rates of 94.7\% and 98.6\%. In addition, unlike the one-dimensional case, NExT also managed to have comparable results to LS in $\Omega$. In Fig. \ref{fig:spherical_harmonics}, an illustration of the errors encountered by both NExT and LS in a specific 2-degree (9 basis elements) polynomial is presented. Notably, NExT demonstrated superior extrapolation results, with its errors barely visible. In contrast, LS exhibited a noticeable bright red area on top of the sphere, indicating a less favorable outcome. In Fig. \ref{fig:spherical_harmonics_real_function}, the true function is plotted alongside the predicted functions using NExT and LS. LS-predicted function resembles the true function, but clear distinctions can be seen, whereas NExT's predicted function is challenging to distinguish from the true function, as evident in the error figure as well.

\section{Conclusions}

We propose a novel Neural Extrapolation Technique (NExT) framework for extrapolation using a neural network. Motivated by neural networks' ability to approximate, we turn the extrapolation problem into an approximation problem. Specifically, our framework uses learning to bypass the drawback of classical and modern methods of extrapolation, which naturally focus on data available over the learning area $\Omega$ and not on the extrapolation area $\Xi$. Given prior information in the form of known function space, the neural network learns to extrapolate by approximating the learning area; the neural network will learn functions from the function space and produce the projection of any input function onto the learned space.

In this study, we have analyzed the difference between obtaining a solution that minimizes an error function over the extrapolation domain and fitting the data in its original domain and then extrapolating it. We have established a connection between the extrapolation and approximation errors and determined a condition number that relates the two via a bound. This condition number indicates the difficulty level of the extrapolation problem based on the specific settings and under the generally accepted approach that one should fit the extrapolation model to the data domain. In contrast, we designed the NExT framework so it does not extrapolate directly by fitting the data, and thus, it circumvents cases of high values of this condition number, which enables it to achieve better results in challenging settings.

This paper presents two versions of the extrapolation problem, each with two types of data priors. These data priors include known subspace and anchor functions--functions that are expected to be in the vicinity of the target function that we want to extrapolate. These settings enable researchers to use the NExT framework in various scenarios. Furthermore, the general definitions allow for two other challenging aspects of extrapolation. The first is extrapolating in a far domain from where data is collected. In this case, exploiting any given prior is crucial, and that is where the learning solution excels. The second aspect is the ability to easily adjust the extrapolation for general domains, particularly manifolds. We demonstrate this application in the numerical part, indicating the NExT framework's strong applicability. Other numerical illustrations also show the robustness of our approach and its advantages.

\section*{Code Availability}
The code used in this paper is available at:
\href{https://github.com/guyhay94/mainfold_extrapolation}{Code}

\section*{Acknowledgment}
The authors greatly thank Prof. Shai Dekel for many insightful discussions and for sharing his thoughts with us during the development of this research.

NS is partially supported by the NSF-BSF award 2019752 and the DFG award 514588180.
\bibliography{main}

\end{document}